\documentclass{article}

\usepackage{wrapfig}
\usepackage{lipsum}
\usepackage{hhline}
\usepackage[utf8]{inputenc} 
\usepackage[T1]{fontenc}    
\usepackage{hyperref}       
\usepackage{url}            
\usepackage{booktabs}       
\usepackage{amsfonts}       
\usepackage{nicefrac}       
\usepackage{microtype}      
\usepackage{xcolor}         
\usepackage{titletoc}
\usepackage{cite}
\usepackage[font=small]{caption}
\newcommand{\colorcolsA}[1]{\cellcolor[rgb]{0.85,0.9,0.9}#1}

\usepackage{algorithm}
\usepackage{algorithmic}

\title{Joint Continual Learning of Local Language Models and Cloud Offloading Decisions with Budget Constraints}

\usepackage{tabularx}
\makeatletter
\newcommand{\thickhline}{%
    \noalign {\ifnum 0=`}\fi \hrule height 1pt
    \futurelet \reserved@a \@xhline
}
\newcolumntype{"}{@{\hskip\tabcolsep\vrule width 1pt\hskip\tabcolsep}}
\makeatother

\usepackage[none]{hyphenat}  %

\newcommand{\footnoteremember}[2]{%
\footnote{#2}
\newcounter{#1}%
\setcounter{#1}{\value{footnote}}%
}
\newcommand{\footnoterecall}[1]{%
\footnotemark[\value{#1}]
}

\author{Evan Chen\,\footnoteremember{equal}{These authors contributed equally to this work.} \footnoteremember{purdue}{Elmore Family School of Electrical and Computer Engineering, Purdue University, West Lafayette, IN} ~~~~~~~~~
Wenzhi Fang\,\footnoterecall{equal} \footnoterecall{purdue} ~~~~~~~~~
Shiqiang Wang\,\footnoteremember{exeter}{Department of Computer Science, University of Exeter, UK}\\
Christopher Brinton\,\footnoterecall{purdue}
\date{}
}

\usepackage{colortbl}
\usepackage{microtype}
\usepackage{graphicx}
\usepackage{subfigure}
\usepackage[numbers,sort]{natbib}

\usepackage{booktabs} 
\usepackage{xcolor}
\usepackage{multirow}
\usepackage{makecell}

\usepackage{hyperref}
\usepackage{enumitem} 
\usepackage{subcaption} 

\usepackage{amsmath}
\usepackage{amssymb}
\usepackage{mathtools}
\usepackage{amsthm}
\usepackage[capitalize,noabbrev]{cleveref}

\theoremstyle{plain}
\newtheorem{theorem}{Theorem}[section]
\newtheorem{proposition}[theorem]{Proposition}
\newtheorem{result}[theorem]{Result}

\theoremstyle{definition}

\theoremstyle{remark}

\usepackage{pifont}

\allowdisplaybreaks

\definecolor{lightred}{rgb}{1, 0.7, 0.4}
\definecolor{lightblue}{rgb}{0.7,0.7,1}

\definecolor{lightgreen}{rgb}{0.7, 1, 0.7}

\usepackage[textsize=tiny]{todonotes}

\begin{document}

\maketitle
\begin{abstract}

Locally deployed Small Language Models (SLMs) must continually support diverse tasks under strict memory and computation constraints, making selective reliance on cloud Large Language Models (LLMs) unavoidable. Regulating cloud assistance during continual learning is challenging, as naive reward-based reinforcement learning often yields unstable offloading behavior and exacerbates catastrophic forgetting as task distributions shift. We propose DA-GRPO, a dual-advantage extension of Group Relative Policy Optimization that incorporates cloud-usage constraints directly into advantage computation, avoiding fixed reward shaping and external routing models. This design enables the local model to jointly learn task competence and collaboration behavior, allowing cloud requests to emerge naturally during post-training while respecting a prescribed assistance budget. Experiments on mathematical reasoning and code generation benchmarks show that DA-GRPO improves post-switch accuracy, substantially reduces forgetting, and maintains stable cloud usage compared to prior collaborative and routing-based approaches.
\end{abstract}

\section{Introduction}

As large language model (LLM) services are increasingly deployed on consumer devices and wireless networks, both \emph{computational} and \emph{economic} considerations motivate pushing intelligence toward the network edge. In practice, such deployments rely on Small Language Models (SLMs) that operate under strict constraints on memory, computation, and energy~\cite{zhu2024llava,zhang2024tinyllama,liu2024mobilellm,xu2024device}. While efficient for routine queries, these compact models struggle with complex reasoning and are vulnerable to catastrophic forgetting when adapted sequentially to diverse tasks. In contrast, cloud LLMs provide strong generalization and reasoning capabilities, but incur substantial communication and monetary costs.

This fundamental asymmetry has motivated \emph{local-cloud collaboration}, where a local SLM answers queries by default and selectively invokes a cloud LLM only when necessary~\cite{xu2024edgellm}, as illustrated in Figure~\ref{fig:framework}. Cloud inference remains expensive: generating a single token on state-of-the-art models (e.g., GPT-5, Gemini Ultra, Claude Opus) is orders of magnitude more costly than running compact models locally. As demand for ubiquitous LLM access grows, providers face sustained pressure to reduce cloud-side cost and energy consumption. Selective offloading therefore offers substantial system-level efficiency gains, but introduces a key requirement: cloud assistance must be invoked sparingly and remain predictable, since each invocation directly translates to monetary cost, latency, and backend capacity usage. As a result, practical systems enforce explicit or implicit offloading budgets rather than allowing unconstrained or drifting collaboration behavior~\cite{belcak2025small,magister2023teaching}.

Crucially, local SLMs in such deployments are not trained once and frozen. Instead, they must \emph{continually adapt} to evolving user workloads and new applications, such as domain-specific question answering or personalized instruction following. Continual learning is therefore unavoidable on local devices. However, unlike cloud-scale models, capacity-limited SLMs are highly susceptible to catastrophic forgetting, where adapting to new tasks rapidly erodes prior capabilities. Unlike cloud-scale deployments, local devices cannot readily store or manage multiple task-specific adapters or model checkpoints due to tight memory and storage constraints, making many standard continual-learning remedies impractical in this setting. This creates a fundamental tension: the local model must improve on new tasks while preserving past knowledge, all while maintaining stable and budget-aware reliance on cloud assistance.

\begin{figure*}
    \centering
    \includegraphics[width=0.95\linewidth]{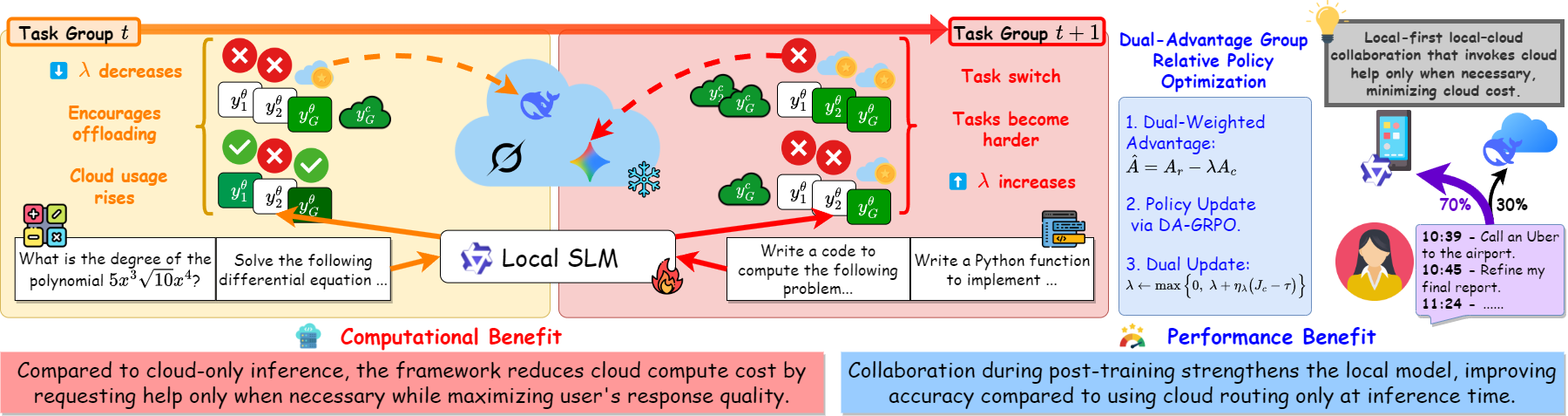}
    \vspace{-0.1in}
    
    \caption{Overview of DA-GRPO: the local SLM produces responses, requests cloud help only when needed, and updates via dual-weighted GRPO, reducing cloud cost and strengthening local performance, while preserving task capabilities.}
    \label{fig:framework}
\end{figure*}

Recent work has explored learning-based collaboration, either by training external routers~\cite{ong2024routellm} or by integrating cloud-offloading decisions directly into local SLMs through reinforcement learning (RL)~\cite{fang2025collaborative}. However, these approaches are not designed for continual learning settings, where task distributions evolve and collaboration preferences must adapt without per-task reward re-tuning. Treating collaboration as a fixed routing or reward-design problem overlooks how offloading decisions shape gradient statistics and update stability during training, which in turn govern knowledge retention and forgetting in capacity-limited models. These observations motivate the following research questions.

\textbf{How can collaboration-aware learning mitigate catastrophic forgetting?} Capacity-limited models suffer catastrophic forgetting in continual learning largely because they are forced to fit difficult or out-of-scope samples. When collaboration decisions are poorly aligned with the model’s evolving capability, limited capacity is wasted on unsolvable problems, amplifying destructive interference across tasks. In contrast, if a local SLM can selectively offload such samples during training, it can concentrate its learning effort on problems it can reliably solve, fundamentally reshaping the optimization trajectory and improving knowledge retention under task transitions.

\textbf{How can collaboration-aware post-training remain effective under task switches and continual learning?}
Most existing approaches either decouple collaboration from post-training or rely on hierarchical reward shaping that can be tuned to work under a fixed task distribution. While such designs may successfully regulate collaboration within a single task, they become difficult to maintain when the task distribution changes. Under task switches, the relative importance of local accuracy and cloud assistance shifts abruptly, yet static reward coefficients and normalized policy-gradient updates are unable to immediately adapt to new incoming information. This results in unstable collaboration behavior or delayed adjustment that requires manual re-tuning, making existing approaches ill-suited for continual settings.

\subsection{Our Contributions.} 
In this work, we investigate how to enable local SLMs to learn task competence and collaboration behavior jointly, under an explicit cloud usage budget. We propose Dual-Advantage Group Relative Policy Optimization (DA-GRPO), a constrained RL framework that treats cloud assistance as a learnable resource constraint rather than a fixed reward signal. Instead of scaling gradients or introducing external routers, DA-GRPO embeds constraint feedback directly into the group-relative advantage computation, allowing collaboration behavior to emerge naturally from training. Our main contributions are as follows:
\begin{itemize}[leftmargin=*, itemsep=0pt, topsep=1pt]

\item
We identify a fundamental limitation of existing RL-based local-cloud LLM collaboration under continual learning: modeling cloud usage through fixed hierarchical rewards requires task-specific tuning and fails to maintain stable collaboration behavior across task switches.
We address this by formulating cloud usage as an explicit resource constraint, enabling budget-aware collaboration without per-task reward redesign (Sec.~\ref{ssec:challenge}).
\item

We propose {\tt DA-GRPO}, which introduces a dual-advantage design that separates task reward and cloud usage into two group-relative advantages and recombines them via a dual variable at the advantage level, enabling collaboration behavior to adapt smoothly under task shifts while maintaining stable cloud budgets, \emph{without requiring a separate router or classifier} (Sec.~\ref{ssec:proposed_method}).

\item
We conduct extensive continual-learning experiments on mathematical reasoning, code generation, and knowledge benchmarks (Sec.~\ref{sec:experiments}). DA-GRPO consistently improves post-switch accuracy, reduces forgetting, and maintains stable cloud usage, outperforming GRPO-, GAPG-, and GVPO-based collaborative baselines.

\end{itemize}

\begin{figure}[t]
    \centering
    \includegraphics[width=0.95\linewidth]{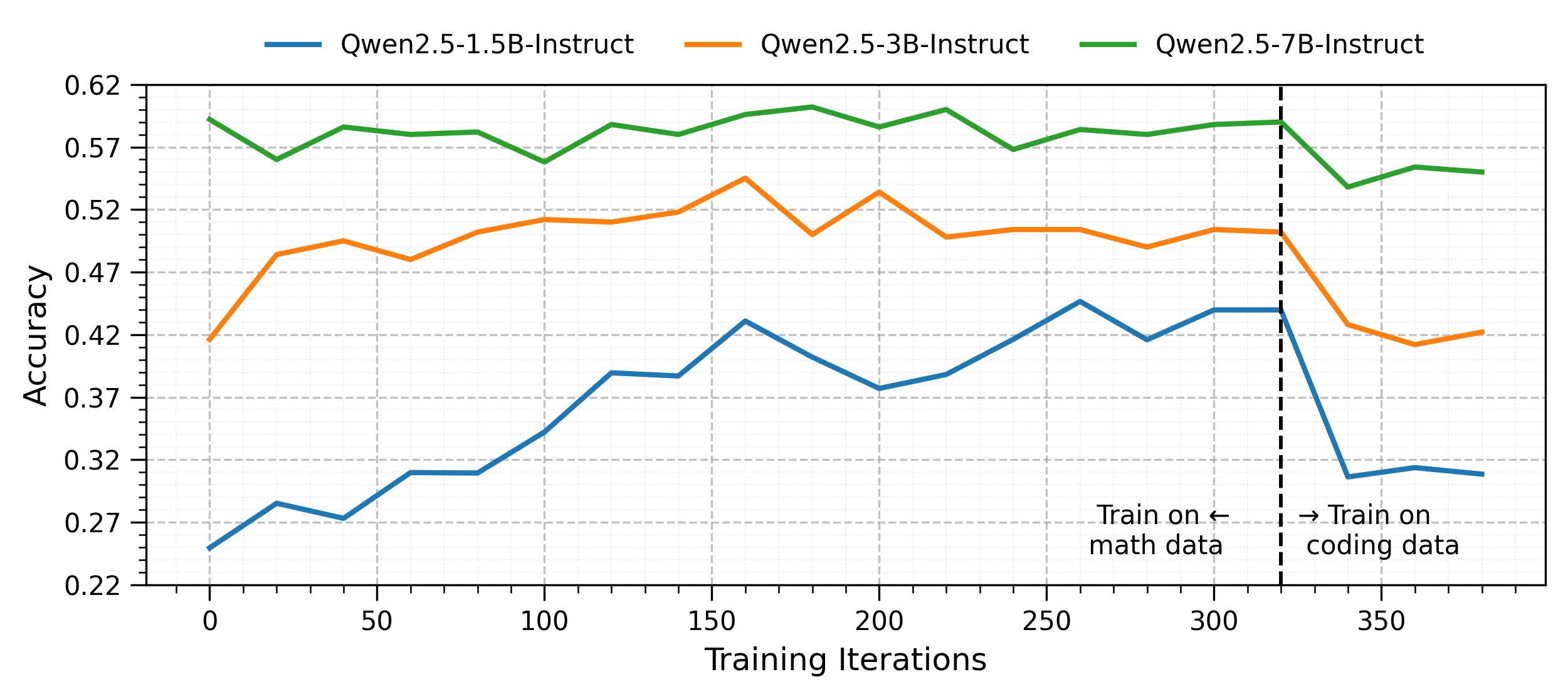}
    \vspace{-0.2in}
    \caption{
    \textbf{Catastrophic forgetting in local SLMs.}
    Sequential fine-tuning results for Qwen2.5-1.5B, 3B, and 7B models evaluated on MATH-500 dataset. 
    After switching to a new task, smaller models (1.5B, 3B) show significant drops in performance on the previous task, while the larger 7B model retains most of its ability. 
    This illustrates the severe forgetting issue for compact LLMs deployed on resource-constrained local devices.
    }
    \label{fig:forgetting}
\end{figure}

\subsection{Related Works.} 

Local-cloud collaboration has emerged as a practical paradigm for deploying large language models under strict resource constraints, where a local SLM handles most queries and selectively relies on a cloud LLM for more challenging cases. Existing systems typically rely on explicit routing mechanisms, such as binary classifiers, confidence-based rules, or external routers~\cite{ding2024hybrid,ong2024routellm,chen2023frugalgpt,oh2025repic}, to decide when queries should be offloaded. While effective at inference time, these approaches decouple routing from model adaptation: the local SLM is trained independently of the collaboration mechanism, leading to post-hoc offloading behavior that cannot be shaped jointly with the model’s learning process and forces uniform capacity allocation across queries of varying difficulty.

More recent work integrates cloud-offloading decisions directly into training using RL, allowing the local model to jointly learn how to respond and when to request cloud assistance~\cite{fang2025collaborative}. By coupling collaboration decisions with task optimization, this paradigm enables collaboration behavior to co-evolve with model capability, which is particularly important under non-stationary or multi-task workloads. In parallel, recent advances in RL-based LLM post-training have introduced group-relative and variance-reduced policy optimization methods, such as GRPO and its variants~\cite{xiong2023iterative,guo2025deepseek,liu2025understanding,zhang2025gvpo}, which improve stability and sample efficiency in unconstrained alignment and reasoning tasks. When applied to local-cloud collaboration, however, these methods typically rely on manually designed reward terms to encourage or penalize cloud usage~\cite{ding2024hybrid, hu2025explicit}.

Classical constrained policy optimization methods enforce resource constraints through Lagrangian updates at the level of the policy objective or gradient~\cite{achiam2017constrained,tessler2019reward,bhatia2019resource,wang2023enforcing}, and recent Lagrangian-based extensions of GRPO follow a similar principle~\cite{gao2025concise}. In networked local-cloud settings, however, the expected constraint objective often depends on discrete routing and execution paths, making it non-differentiable or expensive to approximate, and thus difficult to deploy in practice. 

These limitations motivate collaborative learning frameworks in which a local SLM jointly learns task competence and help-seeking behavior during post-training, while respecting a user-specified collaboration budget. Our work follows this direction by enabling offloading decisions to emerge naturally from the learning dynamics, rather than being imposed by external routers or static reward tuning, leading to more stable collaboration and improved knowledge retention across tasks.

\vspace{-0.05in}
\section{Motivation}
\label{sec:motivation}
\vspace{-0.05in}

To illustrate the forgetting issue for local SLMs on learning multiple tasks sequentially, we conduct a simple continual-training experiment on \textit{Qwen2.5-1.5B}, \textit{Qwen2.5-3B}, and \textit{Qwen2.5-7B} models. Each model is sequentially fine-tuned on two datasets, first a math task and then a code generation task. As shown in Figure~\ref{fig:forgetting}, the smaller models (1.5B and 3B) suffer sharp performance drops on the math performance once the second task is introduced, while the larger 7B model preserves much of its previous capability. This result highlights that smaller local models exhibit a stronger \emph{catastrophic forgetting rate}, making naive fine-tuning or multi-task training ineffective for long-term deployment.

Such forgetting is particularly problematic for \textit{resource-limited local environments}, where: 1) Scaling up to larger models (e.g., 7B+) is infeasible due to device constraints. 2) Storing multiple task-specific adapters quickly exhausts local memory. 3) Repeated retraining or full-precision fine-tuning across tasks incurs unacceptable communication and energy overheads.

Consequently, local SLMs require post-training strategies that can (i) preserve prior capabilities across sequential tasks, (ii) adapt to local or personalized workloads, and (iii) remain lightweight enough for on-device deployment. These challenges call for a new approach that can \emph{exploit the communication structure of wireless edge systems} while \emph{regulating cloud usage} and maintaining robust multi-task performance. Crucially, such regulation should enforce a stable collaboration budget, allowing learning dynamics to adapt internally while preserving predictable system behavior.

\vspace{-0.05in}
\section{DA-GRPO: Dual-Advantage Group Relative Policy Optimization}
\vspace{-0.05in}
\label{sec:method}

We consider a continual post-training setting in which a local SLM must sequentially adapt to a stream of tasks while selectively relying on a cloud LLM under a prescribed assistance budget. At each stage, the local model must balance task performance against the communication cost of invoking cloud assistance, while maintaining stable collaboration behavior as task distributions evolve.

\textbf{Problem formulation.}
We consider two signals obtained from each sample: (i) a primary reward $r(x,y)$ that measures task correctness for the generated response, and (ii) a binary cost $c(x,y) \in \{0,1\}$ that indicates whether the cloud LLM was invoked for that query. Rather than combining these signals into a single scalar reward, we treat them as distinct objectives and formulate the learning problem with an explicit constraint on cloud usage.

Formally, let $\pi_\theta$ denote the local model with parameters $\theta$, and let $\pi_c$ denote a fixed cloud model. We consider a continual learning setting with a sequence of tasks indexed by $\ell \in \{1,\dots,L\}$, where each task $\ell$ is associated with a data distribution $\mathcal{D}_\ell$ and a target cloud-assist rate $\tau_\ell$ that is decided by the cloud to limit the user's usage of cloud resources.

We denote by $p_{\text{coll}}(y | x; \theta)$ the induced response distribution that generates a response using $\pi_\theta$ and invokes $\pi_c$ only when cloud is requested by the local model.
For inputs $x \sim \mathcal{D}_\ell$ and responses $y \sim p_{\text{coll}}(\cdot | x; \theta)$, we define the expected task reward
$J_r^{(\ell)}(\theta)
:=
\mathbb{E}_{x \sim \mathcal{D}_\ell}
\mathbb{E}_{y \sim p_{\text{coll}}(\cdot | x; \theta)}
\big[ r(x,y) \big],$
which captures the average task performance under local-cloud collaboration.
Similarly, the expected cloud-assist rate is given by
$J_c^{(\ell)}(\theta)
:=
\mathbb{E}_{x \sim \mathcal{D}_\ell}
\mathbb{E}_{y \sim p_{\text{coll}}(\cdot | x; \theta)}
\big[ c(x,y) \big],$
where $c(x,y)=1$ if the cloud model is used and $c(x,y)=0$ otherwise.

We focus on a setting in which continual learning is performed only on the local SLM, while the cloud LLM remains fixed throughout training. This asymmetry reflects practical edge-cloud deployments, where the cloud LLM is typically a large, centrally maintained system that is costly to update, shared across many users. In contrast, the local SLM is capacity-limited and must continually adapt to evolving user-specific tasks and data distributions. Hence the cloud mainly serves as an auxiliary resource whose usage is regulated by system constraints.

Our overall goal is to maximize performance across tasks while respecting task-specific cloud-usage budgets:
\begin{equation}
\textstyle\max_\theta \sum_{\ell=1}^L J_r^{(\ell)}(\theta)
\quad
\text{s.t.}
\quad
J_c^{(\ell)}(\theta) \le \tau_\ell,
\;\forall \ell.
\label{eq:constrained_objective}
\end{equation}
We consider a continual learning protocol in which tasks are revealed sequentially. At each stage, updates are computed using samples only from the current task $\mathcal{D}_\ell$, without access to data from previous tasks. Thus the optimization process constitutes an online approximation to~\eqref{eq:constrained_objective}.

\vspace{-0.05in}
\subsection{Challenges in Continual Collaboration Learning}
\vspace{-0.05in}

\label{ssec:challenge}

While joint training enables adaptive collaboration, it also introduces significant challenges in reward design and optimization under continual learning. We highlight two major challenges.

\textbf{1. Task-dependent reward tuning is impractical in con-
tinual learning.} Existing RL-based collaboration methods typically rely on hierarchical or additive reward designs that combine task accuracy and cloud usage through fixed coefficients. Although such designs can be stabilized within a single task using group-relative or variance-reduced updates, they require careful per-task tuning. In continual learning, where task difficulty and reward scales change over time, maintaining a consistent collaboration ratio would require repeated manual re-tune of reward functions.

\textbf{2. Static collaboration rewards ignore the model’s evolv-
ing competency.} As the local SLM improves within a task or transitions to a new task, the relative value of local reasoning versus cloud assistance shifts. Fixed reward bonuses or penalties cannot capture this dynamic, leading to outdated advantage estimates and delayed adaptation after task switches. Over time, this misalignment amplifies catastrophic forgetting and destabilizes collaboration behavior.

These challenges suggest that reward shaping is the primary bottleneck for effective local-cloud collaboration under continual learning. This motivates a constraint-based formulation in which collaboration preferences are enforced explicitly and adapt online with training dynamics.

\setlength{\textfloatsep}{10pt}
\begin{algorithm}[tb]
   \caption{{\tt DA-GRPO}: Dual-Advantage Group-Relative Policy Optimization}
   \label{alg:DAGRPO_continual}
\begin{algorithmic}[1]
\small
   \STATE {\bfseries Input:} tasks $\{\mathcal{D}_\ell\}_{\ell=1}^{L}$, task iterations $\{T_\ell\}_{\ell=1}^{L}$, group size $G$,
   target cloud assist rates $\{\tau_\ell\}_{\ell=1}^{L}$, policy LR $\eta_\theta$, dual LR $\eta_\lambda$.
   \STATE {\bfseries Output:} policy parameters $\theta$
   \STATE initialize policy parameters $\theta \leftarrow \theta_{\text{pretrained}}$
   \STATE initialize dual variable $\lambda \leftarrow \lambda_{\text{init}}$
   \hfill {\footnotesize\textcolor{orange}{$\triangleright$ controls collaboration}}

   \FOR{each task $\ell = 1,\dots,L$}
      \FOR{$t = 1$ {\bfseries to} $T_\ell$}
         \STATE sample batch of prompts $\mathcal{B}_\ell \sim \mathcal{D}_\ell$

         \FOR{each prompt $x_\ell \in \mathcal{B}_\ell$}

            \STATE sample $G$ local responses $\{y_{\ell,i}^\theta\}_{i=1}^{G} \sim \pi_\theta(\cdot \mid x_{\ell})$

            \IF{any response in $\{y_{\ell,i}^\theta\}_{i=1}^{G}$ requests cloud}
               \STATE query cloud LLM $\pi_c$ to obtain $y_{\ell}^{c} \sim \pi_c(\cdot \mid x_{\ell})$
               \STATE set $y_{\ell,i} \leftarrow \mathcal{C}\!\left(y_{\ell,i}^\theta, y_{\ell}^c\right)$ for each help response
               \STATE set $y_{\ell,i} \leftarrow y_{\ell,i}^\theta$ for the remaining responses
            \ELSE
               \STATE set $y_{\ell,i} \leftarrow y_{\ell,i}^\theta$ for all $i \in [1,G]$
            \ENDIF

            \STATE evaluate task reward $r_{\ell,i}$ and collaboration cost $c_i$ for responses $i \in [1,G]$
            \STATE compute group-relative reward advantages and cost advantages using \eqref{eq:advantage_definition}
            \STATE form dual-weighted advantages:
            \STATE \hspace{1em}$\hat{A}_{i} = A^r_{i} - \lambda A^c_{i}$
            \hfill {\footnotesize\textcolor{orange}{$\triangleright$ constraint-aware learning signal}}
         \ENDFOR

         \STATE update policy using $\widehat{\nabla_\theta R}_{\lambda}(\theta,\mathcal{B}_\ell)$ from Proposition~\ref{prop:dual_gradient}:
         \STATE \hspace{1em}$\theta \leftarrow \theta + \eta_\theta \widehat{\nabla_\theta R}_{\lambda}(\theta,\mathcal{B}_\ell)$
         \hfill {\footnotesize\textcolor{orange}{$\triangleright$ primal update}}

         \STATE estimate empirical cloud usage on task $\ell$:
         \STATE \hspace{1em}$\hat{J}_{c,\ell} \leftarrow \mathbb{E}_{x_{\ell}\sim\mathcal{B}_\ell}\!\left[\frac{1}{G}\sum_{i=1}^{G} c_i\right]$
         \STATE perform dual update with task target $\tau_\ell$:
         \STATE \hspace{1em}$\lambda \leftarrow \Big[\lambda + \eta_\lambda(\hat{J}_{c,\ell} - \tau_\ell)\Big]_+$
         \hfill {\footnotesize\textcolor{orange}{$\triangleright$ dual update}}
      \ENDFOR
   \ENDFOR

\end{algorithmic}
\end{algorithm}

\vspace{-0.05in}
\subsection{Constraint-Aware Dual-Advantage Design}
\vspace{-0.05in}

\label{ssec:proposed_method}

We now introduce Dual-Advantage Group-Relative Policy Optimization (DA-GRPO), a constraint-aware policy optimization framework designed to address the challenges in Section~\ref{ssec:challenge}. DA-GRPO incorporates collaboration constraints directly into advantage computation, rather than applying dual variables at the level of the policy objective or gradient \citep{achiam2017constrained,gao2025concise}. In GRPO-style methods, the advantage measures the relative quality of a sampled response compared to other responses generated for the same prompt, and serves as the learning signal that determines which behaviors are reinforced during policy updates. By operating on per-sample comparisons rather than absolute rewards, advantage-based updates are more stable in post-training and continual learning settings.

Algorithm~\ref{alg:DAGRPO_continual} summarizes the full DA-GRPO procedure. At a high level, the algorithm samples groups of responses $G$ for each prompt, determines local or cloud-assisted execution at the response level, and computes group-relative signals for task reward and collaboration cost. These signals are combined to form the learning signal used in the GRPO update, then a dual variable is updated to regulate long-term cloud usage. Intuitively, DA-GRPO is designed to let the model \emph{learn when to seek help in the same way it learns how to respond}. Whereas standard GRPO evaluates responses solely by task reward, DA-GRPO additionally accounts for cloud constraint by injecting this information into the same group-relative comparison mechanism.
The key novelty therefore lies not in the sampling or optimization procedure, which follows standard GRPO, but in the construction of the advantage signal: unlike GRPO, which evaluates responses solely by task reward, DA-GRPO augments the advantage with a collaboration cost term, enabling the policy to internalize performance-resource trade-offs through the same stable group-relative update mechanism.

\textbf{Design rationale.}
Although~\eqref{eq:constrained_objective} provides a principled abstraction of the collaboration problem, directly optimizing it using standard constrained or primal-dual methods typically requires differentiating $J_c^{(\ell)}(\theta)$ with respect to model parameters. In our setting, this is particularly challenging, as cloud usage depends on discrete routing decisions and system-level execution paths, making $J_c^{(\ell)}(\theta)$ non-differentiable and expensive to approximate through policy gradients.

Moreover, in continual learning, policy updates are driven by per-sample advantages computed from the task distribution. Embedding constraint information at the level of individual advantages yields a selective learning signal that empirically helps preserve previously learned capabilities across task transitions, whereas global reward scaling or gradient-level penalties tend to react more slowly to task shifts.

Rather than attempting to realize a strict constrained optimization procedure, DA-GRPO seeks a constraint-aware update mechanism that avoids differentiating $J_c^{(\ell)}(\theta)$ altogether, while still allowing the model to internalize the importance of constraint satisfaction during post-training.

\textbf{Dual-advantage construction.}
DA-GRPO achieves this goal by incorporating the constraint signal directly into the per-sample advantage. For each prompt, we compute separate group-relative advantages for reward and cost. Given a group of $G$ responses $\{y_i\}_{i=1}^G$, we define
\begin{align}
    &\textstyle A_i^r = r_i - \frac{1}{G}\sum_{g=1}^G r_g, \notag\\
    &\textstyle A_i^c = c_i - \frac{1}{G}\sum_{g=1}^G c_g, \label{eq:advantage_definition}
\end{align}
where $r_i$ denotes the task reward and $c_i$ denotes the constraint-related cost, such as cloud usage.

We then form a dual-weighted advantage $\hat A_i = A_i^r - \lambda A_i^c,$
where $\lambda \ge 0$ is a dual variable controlling the relative importance of constraint satisfaction. Crucially, $\lambda$ does not scale the policy gradient directly. Instead, it modulates the advantage signal prior to GRPO’s normalization. 
This construction corresponds to lines~18-20 of Algorithm~\ref{alg:DAGRPO_continual}, where reward and cost advantages are computed separately and combined at the advantage level. This design ensures that constraint information is propagated through the same group-relative mechanism used for reward optimization.

\textbf{Group-wise collaboration and cloud querying.}
As shown in lines~8-16 of Algorithm~\ref{alg:DAGRPO_continual}, DA-GRPO samples a group of $G$ responses for each prompt and allows cloud assistance to be triggered at the response level. When any response in the group requests help, the cloud model is queried once and the resulting cloud response is shared across all help-seeking responses in the group. This design enables heterogeneous local-cloud behavior within a single prompt while avoiding redundant cloud queries. At the same time, maintaining a mixed group of local-only and cloud-assisted responses preserves the group-relative structure required for stable advantage estimation.
Here, $\mathcal{C}(\cdot,\cdot)$ denotes a deterministic response composition operator that replaces the cloud-request trigger in a local response with the corresponding cloud completion. The policy is updated using a GRPO-based gradient estimator.
The formulation and derivation are provided in Proposition~\ref{prop:dual_gradient} in Appendix~\ref{appendix:lgrpo-objective}.
As a result, DA-GRPO remains fully compatible with existing GRPO implementations while enabling task-aware constraint regulation at the prompt level.

\textbf{Dual variable update.}
Although DA-GRPO does not rely on a strict Lagrangian objective, we maintain a single dual variable to track constraint satisfaction over time.
Let $\hat J_{c,\ell}$ denote the empirical average constraint value computed over a batch from the current task $\ell$, and let $\tau_\ell$ be its target cloud-assist budget.
We update $\lambda$ using a simple ascent rule
\begin{equation}
    \lambda \leftarrow \bigl[\lambda + \eta_\lambda (\hat J_{c,\ell} - \tau_\ell)\bigr]_+,
    \label{eq:dual_variable_update}
\end{equation}
where $[\cdot]_+$ denotes projection onto the nonnegative reals.
This update requires only scalar feedback on constraint violation and no model gradients.

Conceptually, $\lambda$ serves as a global regulator that adapts online to task switches, increasing the influence of constraint-related advantages when the current task exceeds its cloud budget and relaxing it otherwise.
Because the constraint enters the learning process through advantage shaping rather than gradient scaling, DA-GRPO avoids the instability and computational cost associated with differentiating complex, non-differentiable constraints. A fixed-point interpretation of this update, showing that any interior stationary point enforces the target collaboration budget in expectation, is provided in Proposition~\ref{prop:dual_fixed_point} in Appendix~\ref{sec:appendix_dual_fixed_point}.

\textbf{Discussion.} DA-GRPO introduces a constraint-guided advantage construction built on top of GRPO, operating at the level of per-prompt advantage computation rather than objective formulation. By embedding the dual signal directly into group-relative advantages, the method enables the model to internalize trade-offs between task reward and constraint satisfaction within each group of responses. This design is particularly well suited to settings where constraints are discrete, non-differentiable, or system-dependent, and where strict Lagrangian optimization would be difficult or impractical to realize in practice.

\vspace{-0.05in}
\section{Experiments}
\label{sec:experiments}
\vspace{-0.05in}
\begin{table*}[t]
\caption{Evaluation of Local-Cloud Collaboration Methods Across Tasks under fixed collaboration ratio, with math data set to $\tau_\ell = 0.3$ and code data set to $\tau_\ell = 0.5$.
}
\setlength{\tabcolsep}{1pt}
\label{tab:collab_results}
\centering
\resizebox{0.98\textwidth}{!}{
\small
\begin{tabular}{l l  l  >{\centering\arraybackslash}p{5em}  
                        >{\centering\arraybackslash}p{5em} 
                        >{\centering\arraybackslash}p{5em} 
                        >{\centering\arraybackslash}p{5em} 
                        >{\centering\arraybackslash}p{5em} 
                        >{\centering\arraybackslash}p{5em} 
                        }
\toprule
\multirow{2}{*}[-0.6ex]{
\begin{tabular}{l}
        \textbf{Evaluated}\\
        \textbf{Responses}
    \end{tabular}
    } & \multirow{2}{*}[-0.6ex]{\textbf{Model}} &\multirow{2}{*}[-0.6ex]{\textbf{Method}} & \multicolumn{3}{c}{\textbf{MATH-lighteval} } & \multicolumn{3}{c}{\textbf{TACO-Verified} } \\

\hhline{~~~------}
 & & & During-task Acc. ($\uparrow$) & Post-Switch Acc. ($\uparrow$) & Forgetting Rate ($\downarrow$) & During-task Acc. ($\uparrow$) & Post-Switch Acc. ($\uparrow$) & Forgetting Rate ($\downarrow$)\\
\midrule
 
 & & Cloud LLM Cluster & 98.4& 98.4& - & 96.3 & 96.3 & - \\
\hline

\multirow{14}{*}{
    \begin{tabular}{l}
        \textbf{Local-solved}\\
        \textbf{Responses only}\\
        \textbf{[$y_i^\theta$]}
    \end{tabular}
}
&\multirow{7}{*}{
    \begin{tabular}{l}
        \textbf{Qwen2.5-} \\
        \textbf{1.5B-Instruct}
    \end{tabular}
} & Edge Tuning Only & 55.3 & 44.4 & 19.7& 44.0 & 24.0&45.5\\
 & & Edge Tuning w/ Naive Router & 55.2 & 44.5 &19.4 & 43.7 & 24.2 &44.6\\
 & & Edge Tuning w/ Router & 62.2 & 55.4 & 10.9 &  55.2 & 38.2 & 30.8\\
\hhline{~~-------}

 && Collaborative Training w/ GRPO & 56.0 & 48.1 & 14.1 & 51.8 & 34.6 &33.2\\
 &&  \phantom{Collaborative Training} w/ GVPO & 59.0 & 53.7 & 9.0& 57.4& 39.5 &31.2\\
 &&  \phantom{Collaborative Training} w/ GAPG & 67.8 & 59.2 & 12.7& 61.6 & 43.1 &29.1\\

 & & \colorcolsA{\textbf{Collaborative Training w/ DA-GRPO (Ours)}} 
  & \colorcolsA{77.2} 
  & \colorcolsA{67.7} 
  & \colorcolsA{12.3} 
  & \colorcolsA{65.8} 
  & \colorcolsA{48.5} 
  & \colorcolsA{26.3}\\
\hhline{~--------}

&\multirow{7}{*}{
    \begin{tabular}{l}
        \textbf{Llama-3.2-} \\
        \textbf{3B-Instruct}
    \end{tabular}
} & Edge Tuning Only & 51.6 & 33.2 & 35.7 & 42.7& 29.9 &30.0\\
 && Edge Tuning w/ Naive Router & 50.3 & 32.8 & 34.8& 42.5 & 30.1 &29.2\\
 && Edge Tuning w/ Router & 58.7 & 49.6 & 15.5 & 48.1 & 34.2 &28.9\\
\hhline{~~-------}

 && Collaborative Training w/ GRPO & 53.1 & 43.7 & 17.7& 46.6 & 33.1 &29.0\\
 &&  \phantom{Collaborative Training} w/ GVPO & 62.0 & 54.2 & 12.6 & 50.9& 37.6&26.1\\
 &&  \phantom{Collaborative Training} w/ GAPG & 63.8 &  58.5 & 8.3 & 51.1&  37.8 &26.0\\

 & & \colorcolsA{\textbf{Collaborative Training w/ DA-GRPO (Ours)}} 
  & \colorcolsA{70.3} 
  & \colorcolsA{63.2} 
  & \colorcolsA{10.1} 
  & \colorcolsA{56.5} 
  & \colorcolsA{44.0} 
  & \colorcolsA{22.1}\\
 \midrule

\multirow{14}{*}{
    \begin{tabular}{l}
        \textbf{Local-Cloud}\\
        \textbf{Joint Responses}\\
        \textbf{[$y_i^\theta,y_i^c$]}
    \end{tabular}
}
&\multirow{7}{*}{
    \begin{tabular}{l}
        \textbf{Qwen2.5-} \\
        \textbf{1.5B-Instruct}
    \end{tabular}
} & Edge Tuning Only & 55.3 & 44.4 & 19.7& 44.0 & 24.0 & 45.5\\
 & & Edge Tuning w/ Naive Router & 68.3 & 60.6 & 11.3& 69.8 & 56.7 & 18.7\\
 & & Edge Tuning w/ Router & 73.1 & 61.3 & 16.1 & 75.8 &61.2& 19.2\\
\hhline{~~-------}

 && Collaborative Training w/ GRPO & 73.5 & 59.8 & 18.6& 70.4 & 54.2& 23.0\\
 &&  \phantom{Collaborative Training} w/ GVPO & 72.7 & 63.1 & 13.2& 76.9 & 60.2  & 21.7\\
 &&  \phantom{Collaborative Training} w/ GAPG & 78.1 & 68.5 & 12.3& 82.2 & 66.6 &19.0\\

 & & \colorcolsA{\textbf{Collaborative Training w/ DA-GRPO (Ours)}} 
  & \colorcolsA{84.5} 
  & \colorcolsA{75.5} 
  & \colorcolsA{10.7} 
  & \colorcolsA{84.8} 
  & \colorcolsA{72.8} 
  & \colorcolsA{14.1}\\
\hhline{~--------}

&\multirow{7}{*}{
    \begin{tabular}{l}
        \textbf{Llama-3.2-} \\
        \textbf{3B-Instruct}
    \end{tabular}
} & Edge Tuning Only & 51.6 & 33.2 & 35.7& 42.7& 29.9 & 30.0\\
 && Edge Tuning w/ Naive Router & 64.7 &52.5 & 18.9& 66.8 & 54.8 &17.9\\
 && Edge Tuning w/ Router & 70.3 & 63.1 & 10.2 & 69.7 & 58.6 &15.9\\
\hhline{~~-------}

 && Collaborative Training w/ GRPO & 61.9 & 53.4 & 13.7& 60.8 &48.9 &19.6\\
 &&  \phantom{Collaborative Training} w/ GVPO & 72.2 & 66.5 & 7.9& 66.7& 54.6&18.1\\
 &&  \phantom{Collaborative Training} w/ GAPG & 77.4 & 69.9 & 9.7& 71.6& 58.9 &17.7\\

 & & \colorcolsA{\textbf{Collaborative Training w/ DA-GRPO (Ours)}} 
  & \colorcolsA{80.7} 
  & \colorcolsA{74.8} 
  & \colorcolsA{7.3} 
  & \colorcolsA{74.4} 
  & \colorcolsA{62.3} 
  & \colorcolsA{16.2}\\

\bottomrule
\end{tabular}
}
\end{table*}

\begin{figure*}[t]
    \centering
    \includegraphics[width=\linewidth]{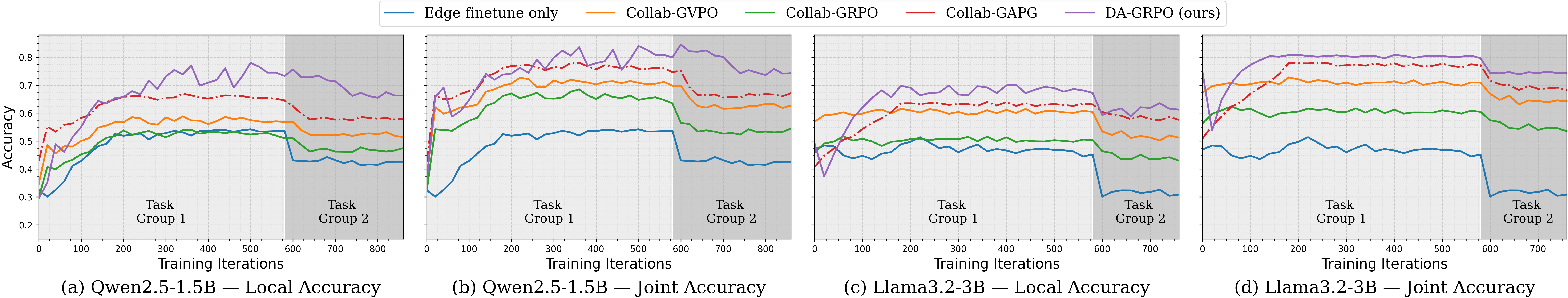}
    \vspace{-0.25in}
    \caption{Testing accuracy on the MATH-lighteval dataset under task switches. Our method exhibits the smallest performance drop after transitioning from math-task training to other task groups.}
    \label{fig:baseline_curves}
\end{figure*}

\textbf{Datasets and Continual Learning Settings.}
We construct a continual learning setting involving heterogeneous datasets spanning math reasoning, code generation, and knowledge reasoning. All experiments follow a multi-phase training protocol in which the model transitions between predefined task groups. We consider the following two task groups:

\begin{enumerate}[leftmargin=*, itemsep=0pt, topsep=1pt]
    \item \textit{Math and Science Reasoning.} This group includes MATH-lighteval, ARC-Easy, and ARC-Challenge~\cite{hendrycks2021measuring, clark2018think}. For MATH-lighteval, we use all available subsets.
    
    \item \textit{Code Generation and General Knowledge.} This group uses the TACO-verified dataset of validated code generations, along with the MMLU benchmark for general knowledge~\cite{hendrycks2020measuring, li2023taco}. For MMLU, we use the ``auxiliary-train'' split and subsample examples to balance its size with TACO-verified.

\end{enumerate}

In the primary experimental configuration, the model is trained for two sequential phases, each corresponding to one task group. We examine two switching scenarios. In the first scenario, training begins with task group~1 and then transitions to task group~2. After completing both phases, we evaluate on MATH-lighteval to quantify forgetting relative to the end of phase~1. In the second scenario, training begins with task group~2 and then transitions to task group~1, after which we evaluate on TACO-verified to measure forgetting on the code generation tasks. Performance on other tasks are provided in Appendix~\ref{appendix-sub:new_benchmarks}

\textbf{Local-cloud LLMs and Collaboration Configurations.}
We evaluate our approach on two SLMs: Qwen2.5-1.5B-Instruct and Llama-3.2-3B-Instruct~\cite{bai2023qwen, touvron2023llama}. The cloud LLM used is Deepseek-R1. We set the primal step size $\eta$ to $10^{-6}$ and the dual step size $\eta_\lambda$ to $10^{-2}$. 
The target collaboration ratio $\tau$ is task dependent: we set $\tau = 0.3$ when the first task group involves math reasoning, and $\tau = 0.5$ when the first task group involves code generation, reflecting the higher reliance on cloud assistance for coding tasks.
The initial dual variable is set to $\lambda_{\text{init}} = 0.5$. A detailed study of different collaboration targets is provided in Appendix~\ref{appendix-sub:different_tau}.

For a fair comparison between non-cloud-assisted and cloud-assisted settings, we report two accuracy metrics. The first metric evaluates performance on a fixed set of problems that are always solved locally by the local model, which reflects improvements in local responses when difficult queries are offloaded. The second metric evaluates the joint response produced through local-cloud collaboration, capturing the overall gains in answer quality when both models contribute. Sensitivity studies on learning rates and collaboration ratios are provided in Appendix~\ref{appendix:sensitivity}.

Importantly, the \emph{local-solved responses only} metric is evaluated on the subset of queries that are executed locally under each method, rather than on a fixed evaluation set.
This metric is therefore not intended to measure the intrinsic capability of the local SLM on the full task distribution.
Instead, it evaluates the quality of problem allocation, i.e., whether a method assigns queries that are solvable by the local model to be handled locally.
As a result, different routing or collaboration strategies induce different local evaluation subsets and may yield different local-only accuracies even when the underlying local model is identical.

\textbf{Baselines.}

We compare our method with two classes of approaches: inference-stage routing without collaborative training, and collaborative training with alternative policy optimization schemes.
For inference-stage routing, we evaluate three configurations. The first configuration uses the local model alone, serving as a lower-bound baseline. The second configuration ({\tt Naive Router}) introduces a simple routing mechanism that randomly selects a subset of queries to be offloaded to the cloud model. The third configuration ({\tt Router}) trains a binary classifier that predicts whether a query should be routed to the cloud model, providing a stronger routing-based baseline. 

For collaborative training, we compare against prior methods that optimize a reward function with hierarchical rewards designed for local-cloud cooperation. To assess the effectiveness of this approach, we evaluate several policy optimization algorithms, GRPO~\cite{guo2025deepseek}, GAPG~\cite{fang2025collaborative}, and GVPO~\cite{zhang2025gvpo}, each applied within the same collaborative training framework. Because these methods require additional hyperparameter tuning for the task-specific reward weights, we perform a grid search to identify the optimal configuration of these optimization methods for every task and collaboration ratio.

\begin{figure}[t]
    \centering
    \includegraphics[width=\linewidth]{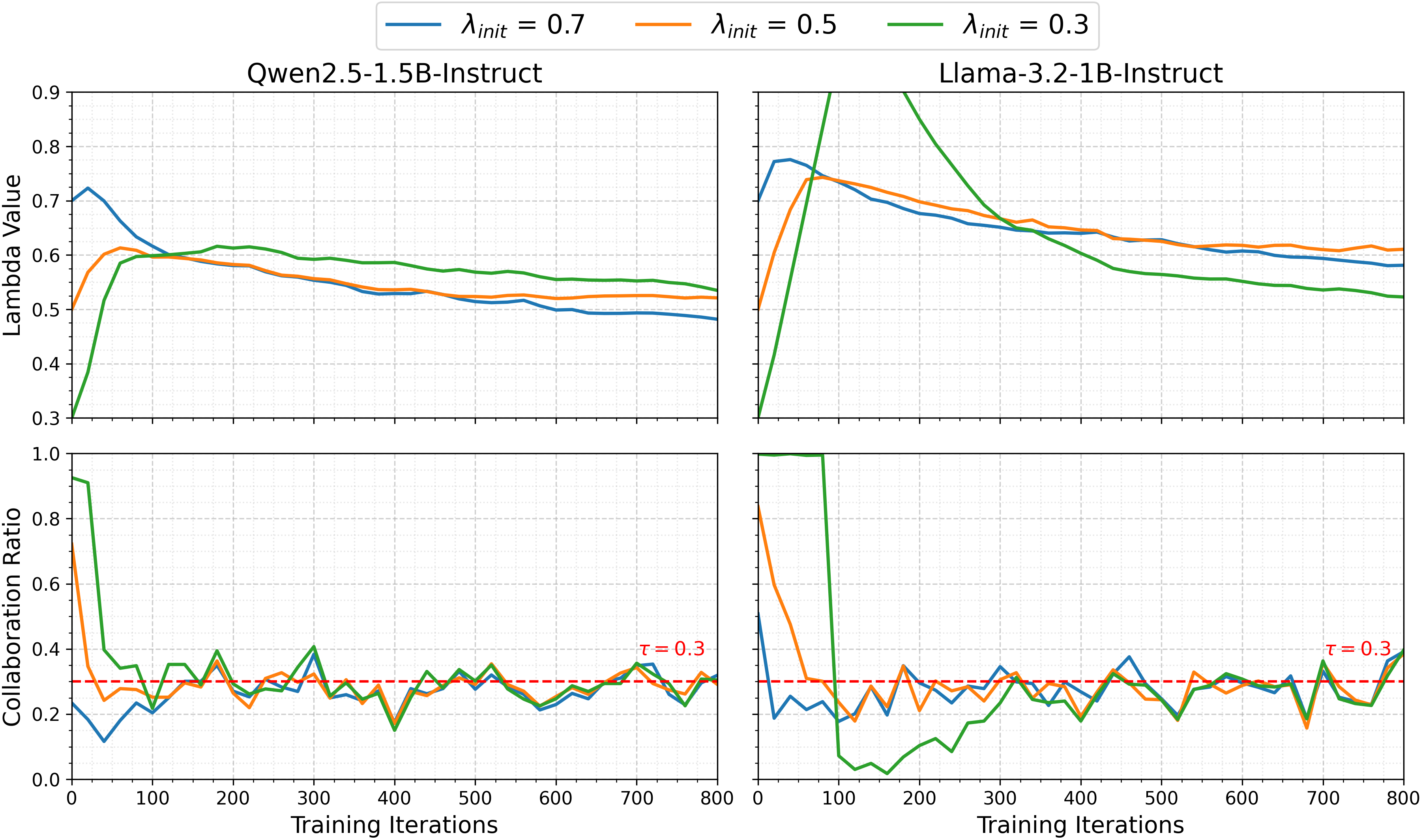}
    \vspace{-0.2in}
    \caption{Comparison of dual variable $\lambda$ trajectories and collaboration ratios for Qwen-2.5B and Llama-3.2B across training iterations for various initialization values.}
    \label{fig:lambda_comparison}
\end{figure}

\begin{figure}[t]
    \centering
    \includegraphics[width=\linewidth]{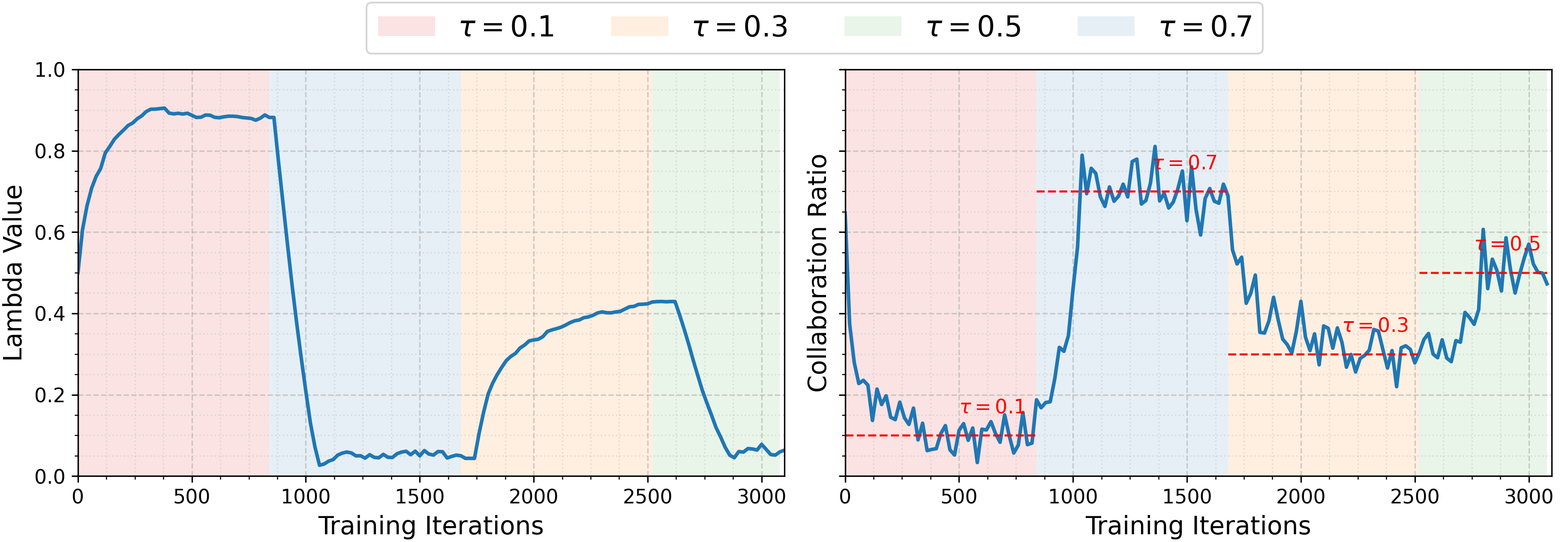}
    \vspace{-0.2in}
    \caption{Adaptive dual advantage under varying collaboration target $\tau$. Background colors indicate different target $\tau$ values. }
    \label{fig:timevary_tau_short}
\end{figure}
\vspace{-0.05in}
\subsection{Experimental Results}
\vspace{-0.05in}

\textbf{Baseline Comparison on Math and Coding Tasks.} 
We evaluate model performance under a continual setting with two task groups. Specifically, we report: 1) during task accuracy when training on the first task group, 2) post switch accuracy after training on the second task group, and 3) the forgetting rate defined as $(Acc_{task} - Acc_{switch})/Acc_{task}$, where $Acc_{task}$ denotes during task accuracy and $Acc_{switch}$ denotes post switch accuracy. Math performance is measured when switching from task group 1 to group 2, while coding performance is measured when switching from task group 2 back to group 1.
Under this evaluation protocol, Table~\ref{tab:collab_results} and Figure~\ref{fig:baseline_curves} show that across both math and coding tasks, DA-GRPO achieves strong gains in joint accuracy and reduced forgetting.
We believe that these improvements are driven by improved problem allocation rather than enhanced local SLM capacity.
By explicitly learning when to offload queries that are difficult for the local model, DA-GRPO avoids expending local capacity on inherently challenging problems, while the cloud model is selectively leveraged where it provides the greatest benefit.

A key limitation of prior fixed hierarchical reward baselines is that the effective collaboration ratio cannot be determined in advance. For a fixed reward design, the final collaboration ratio only emerges after full training and must be rediscovered for each task through repeated grid searches over reward coefficients. This issue is exacerbated under local personalization settings where a large amount of tasks are tuned~\cite{oh2025repic}. Since the optimal collaboration ratio changes as the local model gains task understanding, fixed reward schemes make communication offloading design unnecessarily complex and brittle for real world deployment.

In contrast, DA-GRPO adaptively adjusts collaboration throughout training. When task knowledge is limited, it favors higher collaboration to improve during task accuracy, and as the local model becomes more capable, it reduces cloud reliance, resulting in higher post switch accuracy and lower forgetting. This behavior explains the strong gains observed on math tasks when switching from group 1 to group 2.
The edge tuning baselines with external routers are expected to underperform in this setting, since assistance is explicitly embedded in the reward. This allows DA-GRPO to more precisely select offloading problems based on the local SLM's reasoning capability.

\textbf{Dynamics of $\lambda$ and Sensitivity of Collaboration Ratio.} We then try to answer a key question towards our method: \textit{Under the same task, does the dual variable $\lambda$ converges to reasonable values and stably controls the collaboration ratio?} In Figure~\ref{fig:lambda_comparison}, we plot the values of $\lambda$ starting from various different initialization values under a fixed dataset that mixed {MATH-lighteval} and {ARC} dataset. We can observe that different $\lambda_{init}$ all converges to similar values, leading to a stable collaboration ratio around the target ratio $\tau = 0.3$. Additionally, comparing the results between the Qwen2.5 and Llama3.2, we can observe that the $\lambda$ values are converging towards different values, showing the first advantage against hierarchical unconstrained reward design: DA-GRPO enables flexible deployment of local-cloud collaboration in continual learning settings without the need to tune hyperparameters for every task's reward function.

We further observe that $\lambda$ evolves with training progress: it increases early to discourage over-reliance on the cloud when the model is weak, then decreases as the model gains competence.
This behavior illustrates why enforcing a fixed collaboration budget is critical, as \emph{it provides a stable resource target while allowing the collaboration penalty to adapt automatically with model capability}.

\textbf{Dynamics of $\lambda$ Under Time-Varying Collaboration Ratio.}
We evaluate the behavior of the learned dual variable $\lambda$ under four time-varying collaboration targets $\tau_t$ during continual post-training.
In this experiment, the training dataset is fixed, and only the target collaboration ratio is changed across phases, allowing us to isolate the effect of the dual update from task or data distribution shifts.
Figure~\ref{fig:timevary_tau_short} shows that as $\tau_t$ varies over time, $\lambda$ adjusts correspondingly, tracking the induced changes in collaboration demand and steering the empirical collaboration ratio toward the targets.
These results indicate that the \emph{learned collaboration policy does not overfit to a single target ratio and remains responsive to evolving system-level constraints}, which is key for robust continual learning.
Moreover, this behavior is consistent with our fixed-point interpretation (Proposition~\ref{prop:dual_fixed_point}), which predicts that $\lambda$ adapts to enforce the target budget in expectation.
A more detailed analysis of multiple time-varying collaboration ratios is provided in Appendix~\ref{appendix-sub:time_vary_ratio}.

\vspace{-0.05in}
\section{Conclusion and Future Work}
\label{sec:conclusion}
\vspace{-0.05in}

We presented DA-GRPO, a constrained RL framework for local-cloud collaboration that stabilizes cloud usage while mitigating catastrophic forgetting in continual multi-task settings. By treating cloud assistance as a resource constraint rather than a fixed reward, DA-GRPO enables local SLMs to decide when to solve locally or request cloud help, without relying on external routers or task-specific reward tuning. Extensive experiments demonstrate that DA-GRPO improves post-switch accuracy, reduces forgetting, and maintains stable collaboration ratios across all tasks, making it well suited for resource limited local deployment.

Several directions remain for future work. One promising avenue is to enable flexible, user-adjustable collaboration budget at inference time without retraining. Another is to incorporate auxiliary local signals, such as reasoning traces, to improve the efficiency and quality of cloud inference.

\newpage

\nocite{langley00}

\bibliography{example_paper}
\bibliographystyle{abbrvnat}

\clearpage

\begin{center}
\LARGE \textbf{Appendix}
\end{center}

\startcontents[sections]
\printcontents[sections]{l}{1}{\setcounter{tocdepth}{2}}

\appendix
\onecolumn



\section{Training Objective of DA-GRPO}
\label{appendix:lgrpo-objective}

We first introduce the following result.

\begin{result}[Proposition~3.1 in~\cite{fang2025collaborative}]
\label{result:gradient}
Given a prompt $x$, draw a group of $G$ responses $\{y_1,\ldots,y_G\}$, where each response $y_i$ may be generated entirely by the on-device policy $\pi_\theta$ (i.e., $y_i = y_i^\theta$) or jointly with the cloud policy $\pi_c$ (i.e., $y_i = [y_i^\theta, y_i^c]$). Let $r_i = r(x,y_i)$ denote the reward of response $i$, and define the group mean reward $\bar r = \frac{1}{G}\sum_{i=1}^G r_i$. For any $G \ge 2$, the estimator
\begin{equation}
\widehat{\nabla_\theta R}(\theta,x)
=
\frac{G}{G-1}
\sum_{i=1}^G
\nabla_\theta \log \pi_\theta(y_i^\theta \mid x)\,(r_i - \bar r)
\end{equation}
is an unbiased estimator of the policy gradient
\begin{equation}
\nabla_\theta R(\theta,x)
=
\nabla_\theta \mathbb{E}_{y^\theta \sim \pi_\theta(\cdot \mid x)}\!\left[ r(x,y) \right].
\end{equation}
\end{result}

In our setting, for each prompt $x$, the local policy samples a group of $G$ responses
$\{y_1,\ldots,y_G\} \sim [\pi_{\theta}(\cdot \mid x), \pi_{c}(\cdot \mid x)]$.
We define the group-relative reward and cost advantages as
\begin{equation}
\label{eq:ArAc_appendix}
A_i^r = r_i - \frac{1}{G}\sum_{g=1}^G r_g,
\qquad
A_i^c = c_i - \frac{1}{G}\sum_{g=1}^G c_g,
\end{equation}
and introduce the dual-weighted advantage
\begin{equation}
\label{eq:Ahat_appendix}
\hat A_i = A_i^r - \lambda A_i^c,
\qquad \lambda \ge 0.
\end{equation}

We now show that the resulting update remains an unbiased group gradient estimator under a Lagrangian-shaped reward.

\begin{proposition}[Unbiased dual-weighted group gradient estimator]
\label{prop:dual_gradient}
Define the Lagrangian-shaped reward $r_i' = r_i - \lambda c_i$, and let
$\bar r' = \frac{1}{G}\sum_{i=1}^G r_i'$.
For any $G \ge 2$, the estimator
\begin{equation}
\widehat{\nabla_\theta R}_{\lambda}(\theta,x)
=
\frac{G}{G-1}
\sum_{i=1}^G
\nabla_\theta \log \pi_\theta(y_i^\theta \mid x)\,\hat A_i
\end{equation}
is an unbiased estimator of the policy gradient of the Lagrangian objective
\begin{equation}
\nabla_\theta
\mathbb{E}_{y^\theta \sim \pi_\theta(\cdot \mid x)}
\!\left[ r(x,y) - \lambda c(x,y) \right].
\end{equation}
\end{proposition}

\begin{proof}
Applying Result~\ref{result:gradient} to the shaped reward $r_i' = r_i - \lambda c_i$, we obtain
\begin{align}
\widehat{\nabla_\theta R}_{\lambda}(\theta,x)
&=
\frac{G}{G-1}
\sum_{i=1}^G
\nabla_\theta \log \pi_\theta(y_i^\theta \mid x)\,(r_i' - \bar r') \notag\\
&=
\frac{G}{G-1}
\sum_{i=1}^G
\nabla_\theta \log \pi_\theta(y_i^\theta \mid x)
\Bigl(
r_i - \frac{1}{G}\sum_{g=1}^G r_g
-
\lambda\bigl(c_i - \frac{1}{G}\sum_{g=1}^G c_g\bigr)
\Bigr) \notag\\
&=
\frac{G}{G-1}
\sum_{i=1}^G
\nabla_\theta \log \pi_\theta(y_i^\theta \mid x)\,\hat A_i,
\end{align}
which completes the proof.
\end{proof}

\section{Interpretation of the Dual Variable Update}
\label{sec:appendix_dual_fixed_point}

We analyze the behavior of the dual variable update used in DA-GRPO. Our goal is to clarify why the proposed update can stably regulate the collaboration ratio. We show that, in expectation, any interior fixed point of the dual update corresponds to a policy whose expected cloud usage matches the target budget, providing a fixed-point interpretation of the constraint mechanism.

\begin{proposition}[Fixed point of the dual update enforces target collaboration]
\label{prop:dual_fixed_point}
Fix a task $\ell$ and omit $\ell$ from notation. For a group of $G$ responses, define the group-relative advantages and the dual-weighted advantage with \ref{eq:ArAc_appendix}, \ref{eq:Ahat_appendix}.
Consider the projected dual update
\begin{equation}
\lambda_{t+1}=\bigl[\lambda_t+\eta_\lambda(\hat J_c(\theta_t)-\tau)\bigr]_+,
\label{eq:dual_update_appendix}
\end{equation}
where $\hat J_c(\theta)$ is an unbiased estimator of the expected collaboration cost $J_c(\theta)$.
Assume the primal update approximately optimizes the policy for the current $\lambda_t$.
Then any interior stationary point $\lambda^* > 0$ of the associated mean dual dynamics satisfies
\begin{equation}
J_c(\theta^\star(\lambda^\star))=\tau.\notag
\end{equation}
\end{proposition}

\begin{proof}
We first note that the dual-weighted advantage $\hat A_i$ is equivalent to computing group-relative advantages under the scalar reward
\(
\tilde r(x,y;\lambda)=r(x,y)-\lambda c(x,y).
\)
Indeed,
\[
\tilde r_i-\frac{1}{G}\sum_{g=1}^G \tilde r_g
=(r_i-\bar r)-\lambda(c_i-\bar c)
=A_i^r-\lambda A_i^c
=\hat A_i.
\]

Under our reward/cost function definitions in Appendix~\ref{appendix:template_and_rwd_function}, the scalar reward can be written explicitly as the following hierarchical tuning reward function through the control of $\lambda$:
\begin{equation}
    \tilde r(x,y; \lambda) = \begin{cases}
    \alpha_a \quad \text{If the answer is correct with the correct format and cloud is not used.}\\
    \alpha_a - \lambda \alpha_c \quad \text{If the answer is correct with the correct format and cloud is used.}\\
    -\alpha_f\quad \text{If the answer doesn't follow the step by step format (ignoring any cloud request).}\\
    -\lambda\alpha_c\quad \text{If the cloud answers incorrectly.}\\
    0\quad \text{Otherwise.}
    \end{cases}
\end{equation}

Thus, for fixed $\lambda$, the GRPO-style update using $\hat A_i$ corresponds to optimizing the $\lambda$-weighted reward $\tilde r$. 
Now consider the dual update~\eqref{eq:dual_update_appendix}. 
Under the assumption that the primal update approximately optimizes the policy for the current $\lambda$, the resulting policy can be denoted by $\theta^*(\lambda)$.
Consider the unprojected mean dual dynamics
\[
\bar \lambda_{t+1} = \lambda_t + \eta_\lambda (J_c(\theta^*(\lambda_t)) - \tau).
\]
Any interior stationary point $\lambda^* > 0$ of this mean dynamics must satisfy
\[
J_c(\theta^*(\lambda^*)) - \tau = 0,
\]
In other words, $J_c(\theta^*(\lambda^*)) = \tau$.
The projection operator enforces nonnegativity of $\lambda$ but does not affect interior stationary points. Therefore, any interior stationary point of the projected stochastic update satisfies the stated condition.

\end{proof}

\section{Sensitivity Analysis}
\label{appendix:sensitivity}
\subsection{Analysis on primal step size $\eta$}
\begin{figure}[h]
    \centering
    \includegraphics[width=0.95\linewidth]{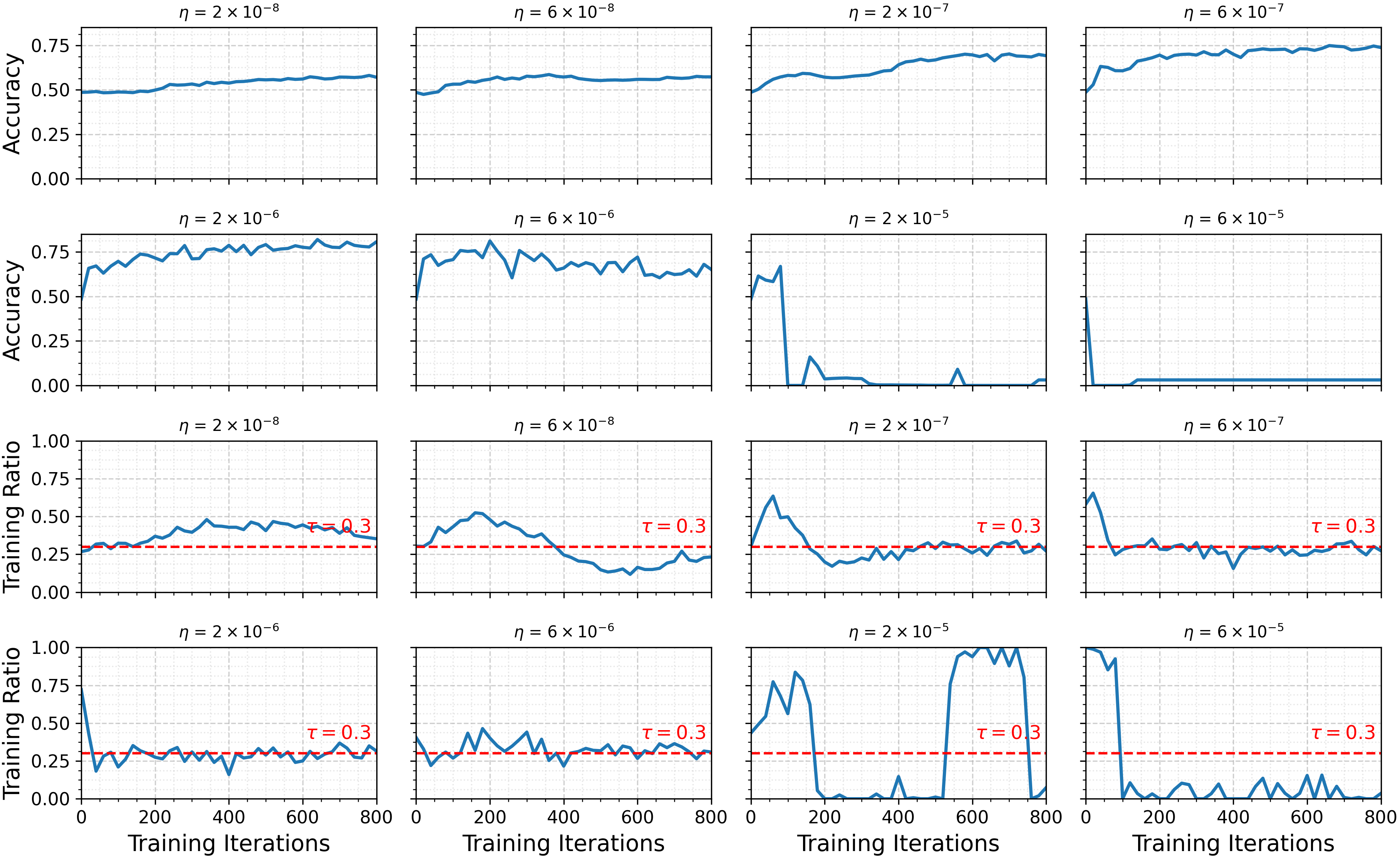}
    \caption{ Accuracy (top two rows) and training ratio (bottom two rows) trajectories over training iterations for different learning rates $\eta$. Moderate learning rates ($\eta = 2\times10^{-6}$ and $6\times 10^{-6}$) yield stable convergence with sustained accuracy and training ratios fluctuating around the target threshold $\tau = 0.3$, while smaller $\eta$ leads to slow or negligible progress and larger $\eta$ causes instability or collapse in both accuracy and training ratio.
}
    \label{fig:stepsize_ablation}
\end{figure}
Fig.~\ref{fig:stepsize_ablation} illustrates the sensitivity of both model accuracy and the training ratio to the learning rate~$\eta$. We observe that extremely small learning rates ($\eta \le 6\times 10^{-8}$) result in negligible progress, where the training ratio remains largely inactive and accuracy stagnates near its initialization level. As $\eta$ increases to a moderate range ($\eta = 2\times 10^{-6}$ and $6\times 10^{-6}$), the system exhibits stable convergence behavior, achieving consistently high accuracy while maintaining the training ratio close to the target threshold $\tau=0.3$. This regime represents a favorable balance between effective learning and controlled training participation.

In contrast, overly large learning rates ($\eta \ge 2\times 10^{-5}$) lead to pronounced instability. In these cases, the training ratio oscillates or collapses to degenerate values, which in turn causes sharp drops in accuracy or complete training failure. These results highlight the importance of proper learning rate selection in our framework, as it directly impacts both optimization stability and the ability to regulate training dynamics according to the desired ratio constraint.

\subsection{Analysis on dual variable step size $\eta_\lambda$}
\begin{figure}[h]
    \centering
    \includegraphics[width=0.95\linewidth]{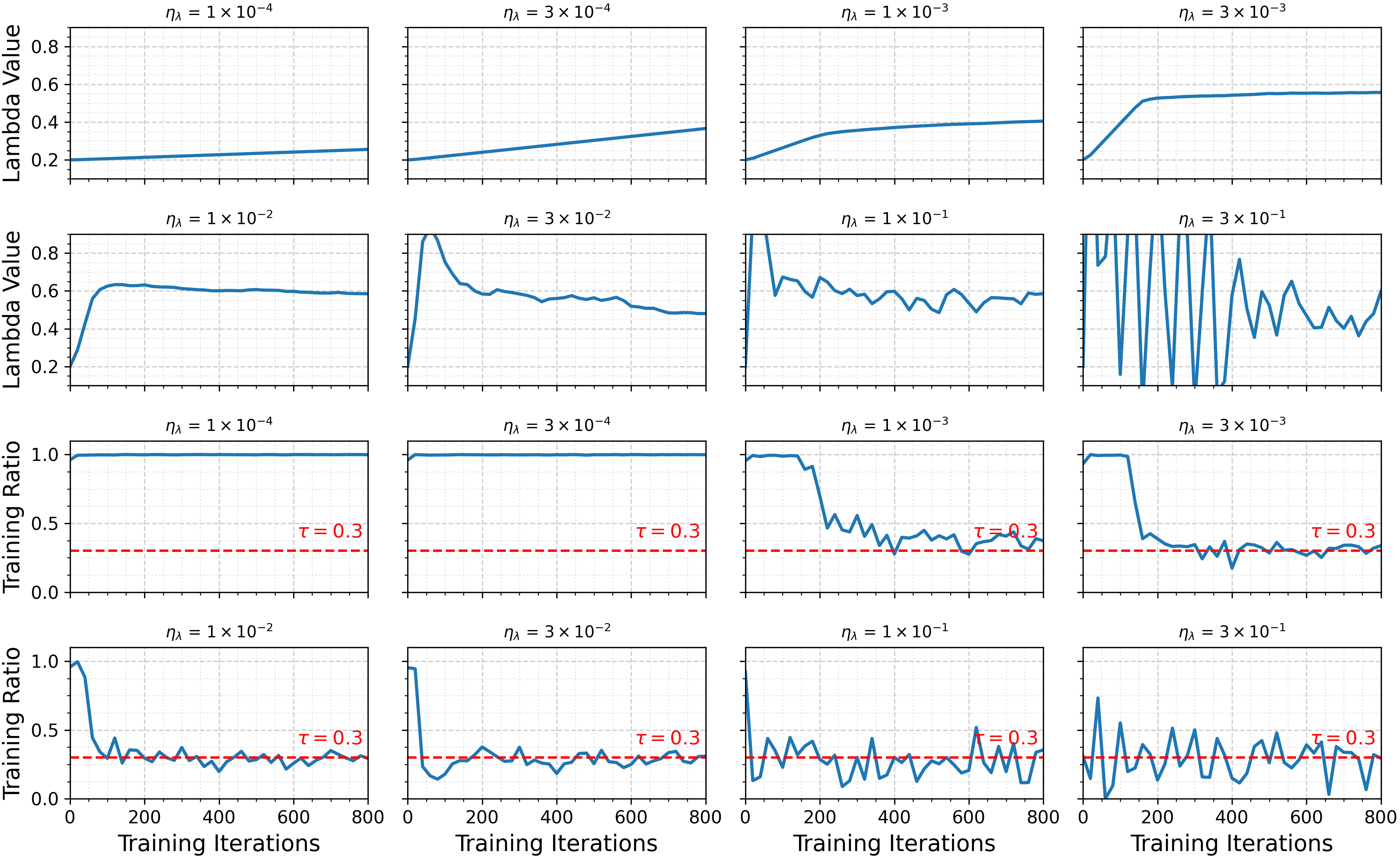}
    \caption{ Evolution of the dual variable \( \lambda \) (top two rows) and the resulting training ratio (bottom two rows) over training iterations for different step sizes \( \eta_\lambda \). Smaller step sizes lead to smooth and stable updates of \( \lambda \), while larger values of \( \eta_\lambda \) induce oscillatory or unstable behavior. The dashed red line indicates the target training ratio threshold \( \tau = 0.3 \), which is approximately maintained for moderate step sizes but violated under overly aggressive updates.
}
    \label{fig:lambda_ablalation}
\end{figure}

Fig.~\ref{fig:lambda_ablalation} studies the impact of the $\lambda$-update stepsize $\eta_\lambda$ on the evolution of the adaptive control variable $\lambda$ and the resulting training ratio on Qwen2.5-1.5B-Instruct. The first two rows show the trajectory of $\lambda$ over training iterations for different values of $\eta_\lambda$, while the last two rows report the corresponding training ratio relative to the target threshold $\tau = 0.3$.

For very small stepsizes ($\eta_\lambda \leq 10^{-3}$), $\lambda$ evolves slowly and remains nearly constant throughout training. Although this regime is stable, it limits the responsiveness of the adaptive control mechanism, effectively reducing it to a static weighting strategy. Consequently, the training ratio converges slowly and even fails to achieve the desired ratio.

When $\eta_\lambda$ is increased to a moderate value ($\eta_\lambda = 1 \times 10^{-2}$), $\lambda$ adapts smoothly and converges rapidly to a stable range without oscillations. In this regime, the training ratio closely tracks the target threshold $\tau$ with low variance across iterations. This behavior reflects a favorable balance between responsiveness and stability, enabling effective correction of communication imbalance while maintaining stable convergence.

For larger stepsizes ($\eta_\lambda \geq 3 \times 10^{-2}$), the dynamics of $\lambda$ become increasingly unstable. Significant oscillations and overshooting are observed, particularly for $\eta_\lambda = 10^{-1}$ and $3 \times 10^{-1}$. These instabilities propagate directly to the training ratio, leading to large fluctuations around the target and, in some cases, extreme deviations. Such behavior undermines convergence stability and degrades the practical effectiveness of the adaptive control mechanism.

Overall, $\eta_\lambda = 1 \times 10^{-2}$ provides the most reliable performance, achieving fast convergence of $\lambda$, stable long-term behavior, and accurate tracking of the desired training ratio. For this reason, this setting is adopted throughout the main experiments unless stated otherwise.

\section{Additional Experiments}
\subsection{Different Collaboration Ratios}
\label{appendix-sub:different_tau}

\begin{figure}[h]
    \centering
    \includegraphics[width=0.95\linewidth]{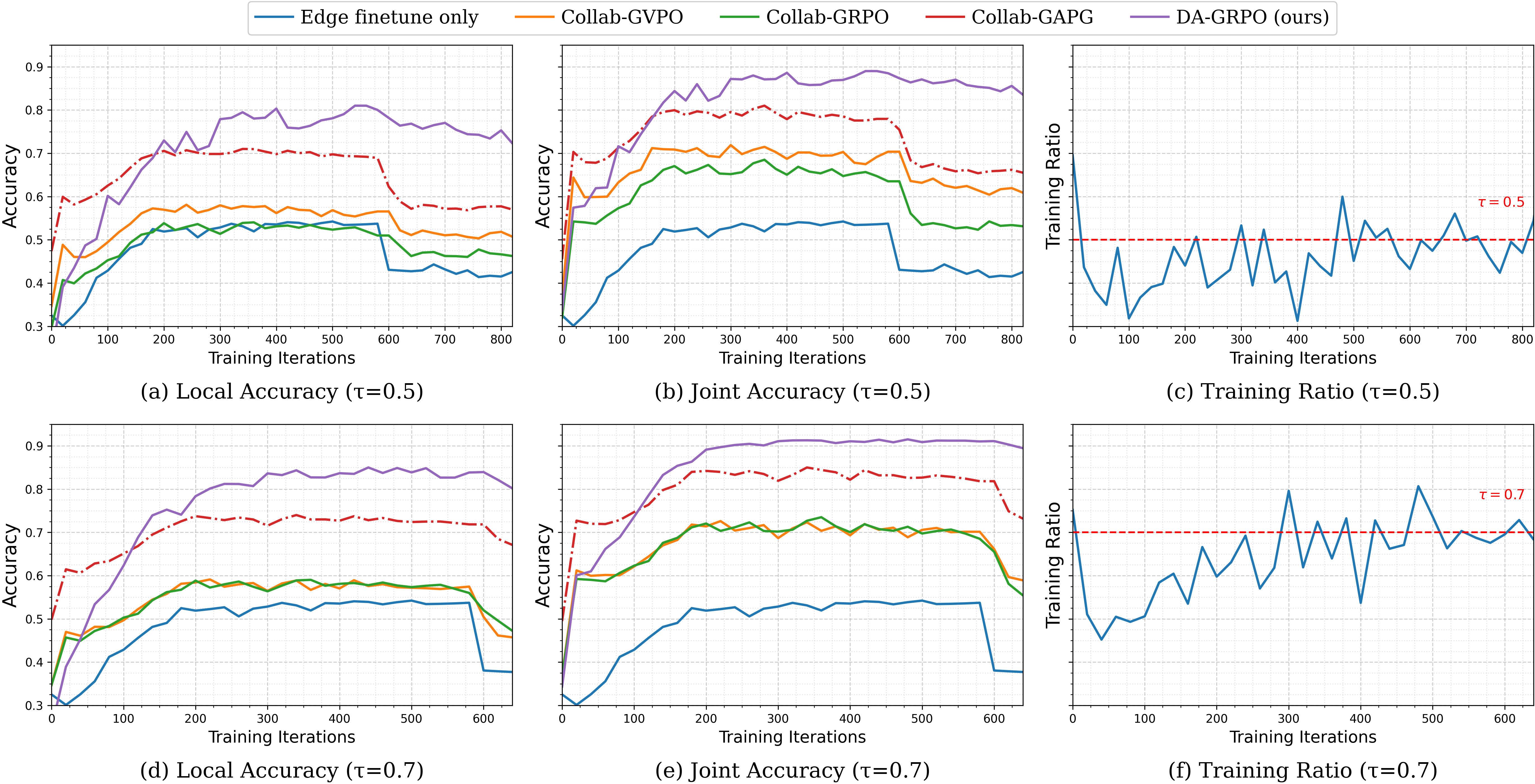}
    \caption{Training dynamics under different collaboration ratios. We report local accuracy, joint accuracy, and the effective training ratio over training iterations for $\tau = 0.5$ (top row) and $\tau = 0.7$(bottom row). DA-GRPO consistently achieves higher local and joint accuracy compared to edge-only fine-tuning and other collaborative baselines, while adaptively adjusting the collaboration frequency to meet the target ratio.}
    \label{fig:appendix_ratios}
\end{figure}
Fig.~\ref{fig:appendix_ratios} illustrates the training dynamics of different methods for Qwen2.5-1.5B-Instruct under collaboration ratios $\tau = 0.5$ and $\tau = 0.7$, including local accuracy, joint accuracy, and the effective training ratio over training iterations. Across both settings, DA-GRPO consistently achieves higher local and joint accuracy than edge-only fine-tuning and other collaborative baselines.

The key reason for this improvement is that the constraint-based dual advantage enables the model to better distinguish between queries that can be reliably solved locally and those that should be offloaded to the cloud. By explicitly accounting for collaboration constraints during optimization, the policy learns an implicit notion of task difficulty and edge capability. As a result, simpler queries are handled locally with higher confidence, while more challenging queries are selectively routed for cloud assistance.

This more informed distribution of tasks leads to simultaneous gains in both local and global accuracy. Moreover, the training ratio curves show that DA-GRPO adaptively tracks the target collaboration budget $\tau$ without requiring manual scheduling or fixed routing heuristics, indicating that the offloading behavior emerges naturally from optimization. Overall, these results demonstrate that capability-aware collaboration is critical for achieving robust and efficient local-cloud coordination.

We then directly compare the final achieved mathematical reasoning performance on the MATH-lighteval benchmark under different collaboration ratios $\tau$. For each model, we report two accuracy metrics:
1) We measure the joint response accuracy, which accounts for both locally generated answers and those delegated to the cloud model. This reflects the overall task performance of the local-cloud system.
2) We compute the local-only accuracy by evaluating only the questions answered directly on the device. This reveals whether cloud assistance improves the quality of local responses by offering a stronger fallback signal during training or alignment.

From Fig.~\ref{fig:collab_ratio}, we can observe that higher collaboration ratios lead to steady improvements in both local and joint accuracies. The local model benefits from delegating harder problems to the cloud model, which allows it to concentrate on questions that fall within its own reasoning capability. This shows that collaboration during the post-training phase can significantly enhance response quality, since the local model is not forced to handle every query uniformly but can instead specialize on the portions of the task it can solve more reliably.

\begin{figure}[t]
    \centering
    \includegraphics[width=0.6\linewidth]{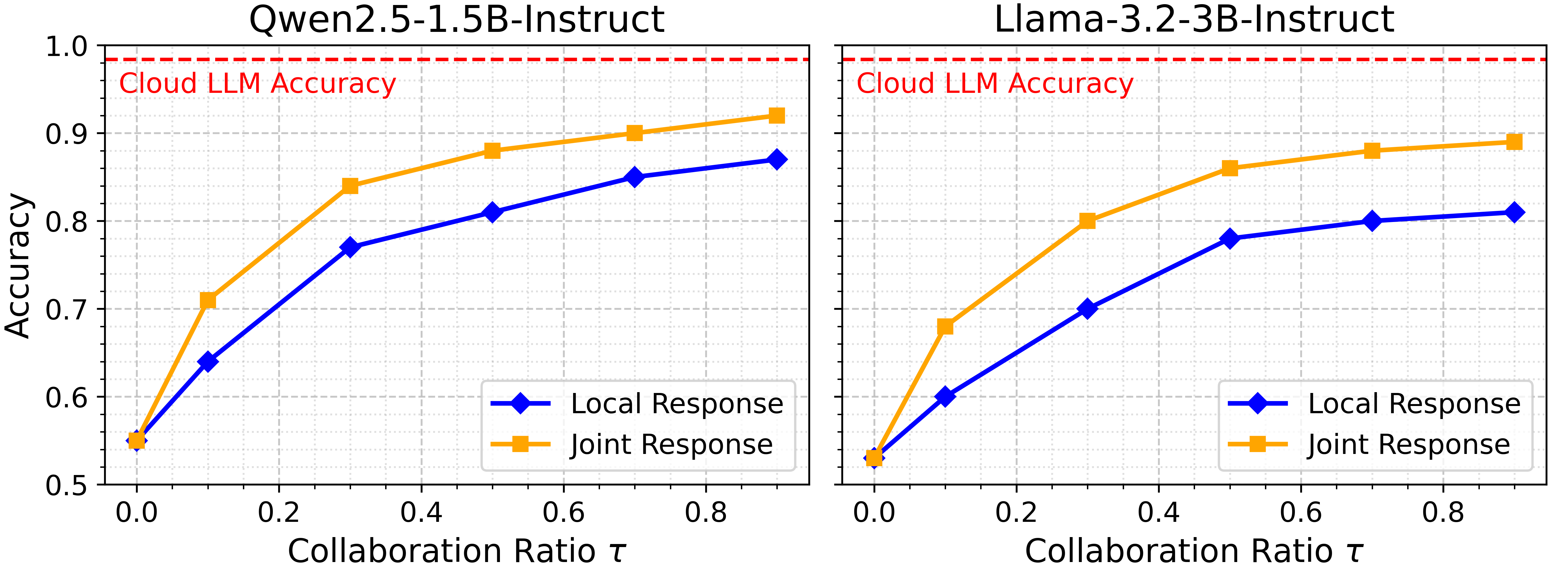}
    \caption{Comparison of edge response accuracy on MATH-lighteval across different cloud assistance ratio. }
    \label{fig:collab_ratio}
\end{figure}

\subsection{Results on additional Benchmarks}
\label{appendix-sub:new_benchmarks}
Table~\ref{tab:appen_benchmarks} reports additional benchmarks across MATH-500 and MMLU, evaluating both local-solved responses and joint local-cloud responses. 
Although the naive router and trained router use the same underlying local SLM, their edge-only accuracies differ because the evaluation subset depends on the routing decisions.
Since the naive router assigns queries randomly, its local subset contains harder problems on average, whereas the trained router learns to retain more locally solvable queries.
We then highlight a few observations across all methods and datasets. First, collaborative training methods outperform local-only tuning across all tasks, highlighting the benefit of leveraging cloud supervision during post-training, even when evaluation is restricted to local-only responses.

Second, among collaborative baselines, methods that rely on fixed routing or static reward formulations (e.g., GRPO, GVPO, GAPG) exhibit noticeable performance degradation after task switches, particularly on reasoning-intensive benchmarks such as MATH-500. This degradation is more pronounced for smaller models, indicating that capacity-limited local models are especially sensitive to reward misalignment under continual learning.

Third, {\tt DA-GRPO} consistently achieves higher post-switch accuracy while maintaining competitive during-task performance across all evaluated settings. The gains are evident for both local-solved responses and joint local-cloud responses, suggesting that the proposed constraint-aware formulation improves collaboration stability without over-reliance on cloud assistance. Notably, these improvements hold across different model families and scales, demonstrating robustness to architectural and capacity differences.

Overall, these additional results corroborate the main findings: incorporating collaboration constraints directly into the learning objective leads to more stable post-training behavior and better retention under task transitions, especially for resource-constrained local models.

\begin{table*}[h]
\caption{Additional Benchmarks of Local-Cloud Collaboration Methods Across Tasks.}
\label{tab:appen_benchmarks}
\centering
\resizebox{0.98\textwidth}{!}{
\small
\begin{tabular}{l l l
                >{\centering\arraybackslash}p{5em}
                >{\centering\arraybackslash}p{5em}
                >{\centering\arraybackslash}p{5em}
                >{\centering\arraybackslash}p{5em}}
\toprule
\multirow{2}{*}[-0.6ex]{
\begin{tabular}{l}
\textbf{Evaluated}\\
\textbf{Responses}
\end{tabular}}
& \multirow{2}{*}[-0.6ex]{\textbf{Model}}
& \multirow{2}{*}[-0.6ex]{\textbf{Method}}
& \multicolumn{2}{c}{\textbf{MATH-500}}
& \multicolumn{2}{c}{\textbf{MMLU}} \\
\hhline{~~~----}
& & 
& During-task Acc. ($\uparrow$)
& Post-Switch Acc. ($\uparrow$)
& During-task Acc. ($\uparrow$)
& Post-Switch Acc. ($\uparrow$) \\
\midrule

& & Cloud LLM Cluster
& 97.3 & 97.3 & 90.8 & 90.8 \\
\hline

\multirow{14}{*}{
\begin{tabular}{l}
\textbf{Local-solved}\\
\textbf{Responses only}\\
\textbf{[$y_i^\theta$]}
\end{tabular}}
& \multirow{7}{*}{
\begin{tabular}{l}
\textbf{Qwen2.5-}\\
\textbf{1.5B-Instruct}
\end{tabular}}
& Edge Tuning Only
& 54.2 & 40.8 & 55.4 & 51.8 \\
& & Edge Tuning w/ Naive Router
& 54.6 & 40.0 & 55.2 & 52.0 \\
& & Edge Tuning w/ Router
& 65.6 & 49.8 & 56.3 & 51.2 \\
\hhline{~~-----}
& & Collaborative Training w/ GRPO
& 58.5 & 43.4 & 57.3 & 54.0 \\
& & \phantom{Collaborative Training} w/ GVPO
& 60.1 & 54.9 & 58.2 & 55.2 \\
& & \phantom{Collaborative Training} w/ GAPG
& 75.2 & 63.3 & 58.9 & 55.6 \\
& & \colorcolsA{\textbf{Collaborative Training w/ DA-GRPO (Ours)}}
& \colorcolsA{77.7} & \colorcolsA{68.1}
& \colorcolsA{59.5} & \colorcolsA{56.2} \\
\hhline{~------}

& \multirow{7}{*}{
\begin{tabular}{l}
\textbf{Llama-3.2-}\\
\textbf{3B-Instruct}
\end{tabular}}
& Edge Tuning Only
& 46.6 & 38.8 & 57.9 & 56.9 \\
& & Edge Tuning w/ Naive Router
& 45.7 & 37.1 & 57.5 & 56.4 \\
& & Edge Tuning w/ Router
& 55.7 & 41.4 & 61.5 & 60.1 \\
\hhline{~~-----}
& & Collaborative Training w/ GRPO
& 52.0 & 40.9 & 59.2 & 57.3 \\
& & \phantom{Collaborative Training} w/ GVPO
& 65.4 & 52.9 & 61.1 & 59.6 \\
& & \phantom{Collaborative Training} w/ GAPG
& 71.1 & 58.3 & 63.1 & 60.6 \\
& & \colorcolsA{\textbf{Collaborative Training w/ DA-GRPO (Ours)}}
& \colorcolsA{73.3} & \colorcolsA{62.1}
& \colorcolsA{64.8} & \colorcolsA{60.7} \\
\midrule

\multirow{14}{*}{
\begin{tabular}{l}
\textbf{Local-Cloud}\\
\textbf{Joint Responses}\\
\textbf{[$y_i^\theta,y_i^c$]}
\end{tabular}}
& \multirow{7}{*}{
\begin{tabular}{l}
\textbf{Qwen2.5-}\\
\textbf{1.5B-Instruct}
\end{tabular}}
& Edge Tuning Only
& 54.2 & 40.8 & 55.4 & 51.8 \\
& & Edge Tuning w/ Naive Router
& 66.4 & 58.0 & 58.9 & 55.3 \\
& & Edge Tuning w/ Router
& 74.8 & 64.0 & 59.8 & 55.2 \\
\hhline{~~-----}
& & Collaborative Training w/ GRPO
& 70.2 & 59.6 & 61.5 & 56.3 \\
& & \phantom{Collaborative Training} w/ GVPO
& 74.0 & 66.2 & 63.1 & 58.5 \\
& & \phantom{Collaborative Training} w/ GAPG
& 81.0 & 72.4 & 64.8 & 59.1 \\
& & \colorcolsA{\textbf{Collaborative Training w/ DA-GRPO (Ours)}}
& \colorcolsA{87.2} & \colorcolsA{77.4}
& \colorcolsA{66.7} & \colorcolsA{59.5} \\
\hhline{~------}

& \multirow{7}{*}{
\begin{tabular}{l}
\textbf{Llama-3.2-}\\
\textbf{3B-Instruct}
\end{tabular}}
& Edge Tuning Only
& 46.6 & 38.8 & 57.9 & 56.9 \\
& & Edge Tuning w/ Naive Router
& 62.0 & 55.8 & 61.2 & 60.3 \\
& & Edge Tuning w/ Router
& 68.2 & 58.2 & 64.2 & 63.2 \\
\hhline{~~-----}
& & Collaborative Training w/ GRPO
& 65.6 & 57.8 & 63.0 & 61.3 \\
& & \phantom{Collaborative Training} w/ GVPO
& 75.0 & 66.2 & 65.4 & 61.8 \\
& & \phantom{Collaborative Training} w/ GAPG
& 79.6 & 70.2 & 67.7 & 62.4 \\
& & \colorcolsA{\textbf{Collaborative Training w/ DA-GRPO (Ours)}}
& \colorcolsA{86.6} & \colorcolsA{76.0}
& \colorcolsA{69.5} & \colorcolsA{63.0} \\

\bottomrule
\end{tabular}
}
\end{table*}

\subsection{Results on Time-varying Collaboration Ratio $\tau_t$}
\label{appendix-sub:time_vary_ratio}

\begin{figure}[h]
    \centering
    \includegraphics[width=0.95\linewidth]{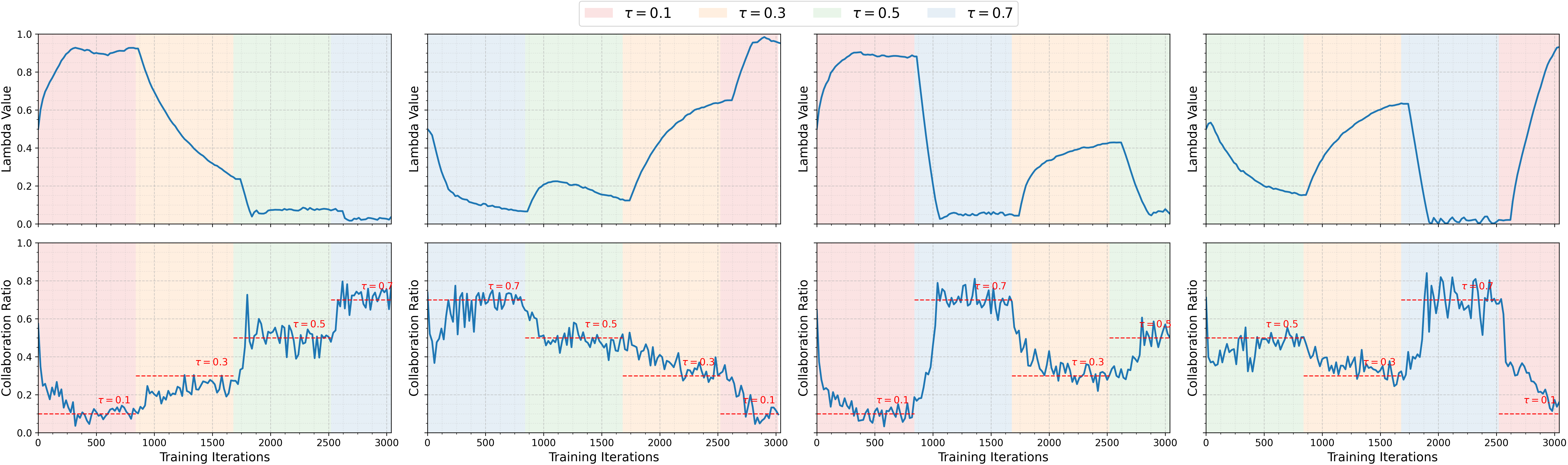}
    \caption{Adaptive post-training behavior under varying collaboration targets $\tau_t$. Top row shows the evolution of the learned dual variable $\lambda$ across training iterations, while the bottom row reports the resulting collaboration ratio. Background colors indicate different target $\tau$ regimes, with dashed horizontal lines marking the desired collaboration ratios.}
    \label{fig:timevary_tau}
\end{figure}

Figure~\ref{fig:timevary_tau} demonstrates the robustness of our method to changes in the target collaboration ratio $\tau_t$ during post-training. 
Rather than fixing the collaboration level through a static penalty or reward formulation, our algorithm continuously updates the dual variable $\lambda$ in response to observed constraint violations. To establish a thorough evaluation through all different scenarios, we observe four different time-varying $\tau$ settings: 1) $\tau_t: 0.1\rightarrow0.3\rightarrow0.5\rightarrow0.7$, 2) $\tau_t: 0.7\rightarrow0.5\rightarrow0.3\rightarrow0.1$, 3) $\tau_t: 0.1\rightarrow0.7\rightarrow0.3\rightarrow0.5$, and 4) $\tau_t: 0.5\rightarrow0.3\rightarrow0.7\rightarrow0.1$.
As $\tau$ varies across different phases, the $\lambda$ adapts accordingly, increasing when the achieved collaboration ratio exceeds the target and decreasing otherwise. 
This mechanism enables stable tracking of the desired collaboration level, as reflected in the bottom row.

A key observation is that the learned policy does not overfit to a single value of $\tau$ specified during training. 
When $\tau$ is altered after convergence to a previous operating point, the system rapidly transitions to the new target without manual hyperparameter tuning. 
This behavior is particularly important in practical local-cloud and multi-agent settings, where collaboration budgets and system constraints may evolve over time. 
Overall, these results validate that the proposed dual-driven design provides flexible and stable post-training control over collaboration behavior.

\section{Details of Hyperparameters}
\definecolor{oursbg}{RGB}{235,245,255}      
\definecolor{baselinebg}{RGB}{205,205,205}  

\begin{table}[t]
    \centering
    \begin{tabular}{c|c}
         \textbf{Hyperparameters} & Values \\
         \hline
         Batch size & 128 \\
         Group size $G$ & 8 \\
         Max prompt length & 1024\\
         Max response length & 1024\\
         Primary Learning rate & $2\times 10^{-6}$\\
         \rowcolor{oursbg}
         Dual Learning rate (ours only) &  $1\times 10^{-2}$\\
         Total Training Steps & 840 (Task Group 1) / 400 (Task Group 2)\\
         Sampling temperature for training & 1.0\\
         Sampling temperature for evaluation & 0\\
         Limit for reasoning steps & 4 (Math) / 6 (Code) / 4 (QA)\\
         \rowcolor{baselinebg}
         Assistance reward (fixed hierarchical methods)* & 
         0.2 (Math) / 0.05 (Code) / 0.4 (QA)\\
    \end{tabular}
    \caption{Training hyperparameters used across all experiments. *Assistance reward tuned for math collaboration ratio 0.3, coding ratio 0.5, and QA ratio 0.2, only used in fixed hierarchical reward methods. }
    \label{tab:hyperparameters}
\end{table}

Table~\ref{tab:hyperparameters} summarizes the training hyperparameters used across all experiments.
For baselines based on fixed hierarchical rewards, the assistance reward must be tuned separately for each task to achieve a desired collaboration ratio.
In our experiments, the reward was manually adjusted from the dual variable design for math, coding, and QA tasks to match their respective target cloud-usage budgets, highlighting a key limitation of fixed reward designs in continual learning settings.

\section{Prompt Template and Reward Function}
\label{appendix:template_and_rwd_function}
\subsection{Prompt Template}

\begin{itemize}
    \item \textit{System prompt:}
    \begin{quote}
You are a helpful reasoning assistant for questioning, code, and math problems.
When a question is posed, try to answer using step-by-step reasoning.
Reason step by step, using format: Step 1: \ldots, Step 2: \ldots, Step 3: \ldots, etc.
You must prefix each reasoning line with `Step k:' exactly (e.g., `Step 1:', `Step 2:').
Do not use other numbering formats such as `1.', `(1)', or `- Step 1'.
If the problem appears difficult or hard to solve, elaborate your reasoning with more detailed steps.
Otherwise, use fewer steps as appropriate.
Enclose your reasoning and answer within \texttt{<think>\ldots</think>} and \texttt{<answer>\ldots</answer>} tags, respectively.
Always box the final answer using \verb|\boxed{}|, e.g., Final answer: \verb|\boxed{42}| (math), \verb|\boxed{choice}| (general QA), or Final answer: \verb|\boxed{result}| (coding or function output).
If after sufficient reasoning steps you determine the problem cannot be solved locally, return: `I need external assistance.'
\end{quote}

    \item \textit{User prompt:} 
    \begin{itemize}
        \item Code: Solve the programming task below in a Python markdown code block. Let's think step by step. \{\textit{question}\}.
        \item Math: \{\textit{question}\}. Let's think step by step.
        \item QA: \{\textit{question}\}. Let's think step by step. Possible answers: \{\textit{choices}\}.
    \end{itemize}
    
\end{itemize}

\paragraph{Rationale for enforcing step-by-step reasoning.}
The requirement that responses follow a step-by-step reasoning format serves not only as a formatting constraint but also as a mechanism to encourage the local SLM to implicitly assess the difficulty of each problem and allocate an appropriate level of reasoning effort. In our implementation, we enforce a maximum number of reasoning steps per task (as shown in Table~\ref{tab:hyperparameters}), such that longer step sequences reflect increased internal computation rather than unbounded verbosity. This design is motivated by empirical findings in prior work \citep{fang2025collaborative}, which show a strong correlation between the number of reasoning steps produced by SLMs and their probability of generating correct solutions. Intuitively, problems that require deeper reasoning tend to elicit longer step sequences, while easier problems can often be solved with fewer steps. By penalizing incorrect or improperly formatted responses, the reward function encourages the model to reserve more extensive reasoning for cases where it is genuinely needed, thereby aligning reasoning effort with task difficulty. We note that while step count provides a practical proxy for difficulty-aware self-regulation, the broader question of how prompt design or output structure can best encourage local models to accurately self-evaluate their success likelihood remains an open direction for future research.

\subsection{Reward Function}
For cloud-assisted responses, the task reward $r_{\ell,i}$ is always evaluated on the finalized composed response
\[
y_{\ell,i} = \mathcal{C}\!\left(y_{\ell,i}^\theta, y_{\ell,j}^c\right).
\]
The local response $y_{\ell,i}^\theta$ is mainly used to determine whether cloud assistance is requested. 

The reward function $r(x,y)$ encodes both answer correctness and formatting compliance. 

\begin{equation}
    r(x,y) = \begin{cases}
    \alpha_a \quad \text{If the answer is correct with the correct format.}\\
    -\alpha_f\quad \text{If the answer doesn't follow the step by step format.}\\
    0\quad \text{Otherwise.}
    \end{cases}
    \label{eq:reward_function}
\end{equation}
Specifically, a positive reward $\alpha_a = 1$ is assigned when the response is correct and follows the required step-by-step format, while a penalty $\alpha_f = 0.1$ is applied if the answer is correct but violates the formatting constraint. 
Responses that are incorrect receive zero reward. 
This design encourages both solution accuracy and adherence to the desired reasoning structure.

The cost function $c(x,y)$ captures the use of cloud assistance. 

\begin{equation}
    c(x,y) = \begin{cases}
    \alpha_c \quad \text{If cloud assistance is requested and the step by step format is enforced.}\\
    0\quad \text{Otherwise.}
    \end{cases}
\end{equation}
A unit cost $\alpha_c = 1$ is incurred whenever cloud assistance is requested, and zero otherwise. 
Together, these definitions enable explicit control over the tradeoff between answer quality and reliance on cloud resources.

\textbf{Correctness Evaluation and Reward Assignment.}
The reward function in Eq.~\ref{eq:reward_function} depends on whether a generated response is deemed correct. During our RL training framework, correctness is determined using task-specific automated evaluation procedures, without employing a learned reward model or human feedback.

For math reasoning benchmarks, correctness is evaluated by comparing the final answer produced by the model with the ground-truth solution.
Specifically, we extract the content enclosed in the \verb|\boxed{}| expression from the model’s response and perform an exact string comparison with the reference answer.
We do not evaluate the validity of intermediate reasoning steps or proof structure; only the final boxed answer is used to determine correctness.

For multiple-choice and general knowledge benchmarks, correctness is similarly evaluated using the final boxed answer.
The model is instructed to output the selected option either as a letter (e.g., A/B/C/D) or as an index (e.g., 1/2/3/4).
We extract this value from the \verb|\boxed{}| expression and perform a string comparison against the ground-truth label provided by the dataset.

For code generation benchmarks, we evaluate correctness through program execution.
We manually curate the TACO-verified dataset to retain only problems with a valid and executable test bench.
The model-generated code is extracted from the Python markdown block in the response and executed in a local containerized environment against all provided test cases.
A response is marked as correct if and only if the generated code executes successfully and passes all test cases without errors, or else it receives zero reward.

Overall, this evaluation protocol ensures that reward assignment is deterministic, automated, and task appropriate.
The reward function reflects task success rather than stylistic quality, and no learned reward model is used in training.

\end{document}